\documentclass[lettersize,journal]{IEEEtai}
\usepackage{amsmath,amsfonts}
\usepackage{algorithmic}
\usepackage{algorithm}
\usepackage{array}
\usepackage[caption=false,font=normalsize,labelfont=sf,textfont=sf]{subfig}
\usepackage{textcomp}
\usepackage{stfloats}
\usepackage{url}
\usepackage{verbatim}
\usepackage{graphicx}
\usepackage{cite}
\usepackage{bm}
\usepackage{makecell}

\usepackage{amsfonts}  
\usepackage{algorithm}
\usepackage{algorithmic}
\usepackage{siunitx}   
\usepackage{upgreek}   
\usepackage{xcolor}    
\usepackage{booktabs}       
\usepackage{longtable}
\usepackage{threeparttable}
\usepackage{multirow}
\usepackage{multicol}
\usepackage{times}
\usepackage{hyperref}

\usepackage{amsthm}

\def\x{\mathbf{x}}

\def\R{\mathbf{R}}

\def\O{\mathcal O}
\def\mN{\mathcal N}
\def\mP{\mathcal P}
\def\mR{\mathcal R}
\def\mA{\mathcal A}
\def\mB{\mathcal B}
\def\0{\mathbf{0}}
\def\v{\mathbf{v}}
\def\x{\mathbf{x}}

\def\S{\mathcal{S}}

\def\rank{\text{rank}}

\usepackage[normalem]{ulem}

\newtheorem{thm}{Theorem}
\newtheorem{lem}{Lemma}

\newtheorem{definition}{Definition}
\newtheorem{pre}{Preliminary}

\def\S{\mathcal{S}}

\def\v{\mathbf{v}}

\def\S{\mathcal{S}}

\def\v{\mathbf{v}}

\def\x{\mathbf{x}}

\def\R{\mathbf{R}}

\def\O{\mathcal O}
\def\mN{\mathcal N}
\def\mP{\mathcal P}
\def\mR{\mathcal R}
\def\mA{\mathcal A}
\def\mB{\mathcal B}
\def\0{\mathbf{0}}
\def\v{\mathbf{v}}
\def\x{\mathbf{x}}

\def\S{\mathcal{S}}

\definecolor{darkblue}{rgb}{0,0.22,0.66}
\hypersetup{colorlinks=true,citecolor=cyan,linkcolor=darkblue}
\usepackage{mathrsfs}
\usepackage{float}
\usepackage{mathtools,amssymb}
\graphicspath{{figuresPF/}{figure/}}

\let\cdots\customcdots

\let\myforall\forall
\def\forall{{\myforall\, }}
\let\myexists\exists
\def\exists{{\myexists\, }}
\let\emptyset\varnothing

\begin{document}
	
\title{\huge\bf Deep ReLU Networks Have Surprisingly Simple Polytopes}

\author{Feng-Lei Fan$^{1}$, \textit{Member, IEEE}, Wei Huang$^{2}$, Xiangru Zhong$^{1}$, Lecheng Ruan$^{3}$, Huan Xiong$^{4}$, Tieyong Zeng$^{1}$, Fei Wang$^{5}$, \textit{Senior Member, IEEE} 
\thanks{*Huan Xiong is the corresponding author. }
\thanks{$^{1}$Feng-Lei Fan, Xiangru Zhong, and Tieyong Zeng are with Center of Mathematical Artificial Intelligence, Department of Mathematics, The Chinese University of Hong Kong, Shatin, Hong Kong.}
\thanks{$^{2}$Wei Huang is with RIKEN Center for Advanced Intelligence Project (AIP), Tokyo, Japan} 
\thanks{$^{3}$Lecheng Ruan is with the College of Engineering, Peking University, Beijing, China.}
\thanks{$^{4}$Huan Xiong is with Institute for Advanced Study in Mathematics, Harbin Institute of Technology, Harbin, Heilongjiang Province, China.}
\thanks{$^{5}$Fei Wang is with Weill Cornell Medicine, Cornell University, New York City, NY, USA.}
}


\markboth{Journal of \LaTeX\ Class Files,~Vol.~14, No.~8, August~2021}%
{Shell \MakeLowercase{\textit{et al.}}: A Sample Article Using IEEEtran.cls for IEEE Journals}

	
\maketitle
	
\begin{abstract}

A ReLU network is a piecewise linear function over polytopes. Figuring out the properties of such polytopes is of fundamental importance for the research and development of neural networks. So far, either theoretical or empirical studies on polytopes only stay at the level of counting their number, which is far from a complete characterization. Here, we propose to study the shapes of polytopes via the number of faces of the polytope. Then, by computing and analyzing the histogram of faces across polytopes, we find that a ReLU network has relatively simple polytopes under both initialization and gradient descent, although these polytopes can be rather diverse and complicated by a specific design. This finding can be appreciated as a kind of generalized implicit bias, subjected to the intrinsic geometric constraint in space partition of a ReLU network. Next, we perform a combinatorial analysis to explain why adding depth does not generate a more complicated polytope by bounding the average number of faces of polytopes with the dimensionality. Our results concretely reveal what kind of simple functions a network learns and what will happen when a network goes deep. Also, by characterizing the shape of polytopes, the number of faces can be a novel leverage for other problems, \textit{e.g.}, serving as a generic tool to explain the power of popular shortcut networks such as ResNet and analyzing the impact of different regularization strategies on a network's space partition.

\end{abstract}

\begin{IEEEImpStatement}
In this work, beyond counting the number of polytopes, we propose to count the number of faces every polytope has for a more complete characterization of ReLU networks. Then, we find that a ReLU network has surprisingly simple polytopes, which is a major generalization of Hanin's famous result that a ReLU network has surprisingly few polytopes. Lastly, via combinatorial techniques, we theoretically derive the tight upper bound for the average face number of polytopes to support our empirical observations. In brief, our work not only provides a new dimension but also a new tool to study the properties of ReLU networks.

\end{IEEEImpStatement}
	
	
\section{Introduction}
\label{submission}

It was shown in a thread of studies \cite{chu2018exact,balestriero2020mad,hanin2019deep,schonsheck2019chart} that a neural network with the piecewise linear activation is to partition the input space into many convex regions, mathematically referred to as polytopes, and each polytope is associated with a linear function (hereafter, we use convex regions, linear regions, and polytopes interchangeably). Hence, a neural network is essentially a piecewise linear function over polytopes. Based on this property, the core idea of a variety of important theoretical advances and empirical findings is to turn the investigation of neural networks into the investigation of polytopes. Figuring out the properties of such polytopes can shed light on many critical problems, which can greatly expedite the research and development of neural networks. Let us use two representative examples to demonstrate the utility of characterizing polytopes:

The first is the explanation of the power of depth. In the era of deep learning, many studies \cite{mohri2018foundations, bianchini2014complexity,telgarsky2015representation,arora2016understanding} attempted to explain why a deep network can perform superbly over a shallow one. One explanation to this question is on the superior representation power of deep networks, \textit{i.e.}, a deep network can express a more complicated function but a shallow one with a similar size cannot \cite{cohen2016expressive,poole2016exponential,xiong2020number}.
Their basic idea is to characterize the complexity of the function expressed by a neural network, thereby demonstrating that increasing depth can greatly maximize such a complexity measure compared to increasing width. Currently, the number of linear regions is one of the most popular complexity measures because it respects the functional structure of the widely-used ReLU networks. 
\cite{pascanu2013number} firstly proposed to use the number of linear regions as the complexity measure. By directly applying Zaslavsky's Theorem \cite{zaslavsky1997facing}, \cite{pascanu2013number} obtained a lower bound $\left(\prod_{l=0}^{L-1} \left \lfloor{\frac{n_l}{n_0}}\right \rfloor\right)\sum_{i=0}^{n_0} \binom{n_L}{i}$ for the maximum number of linear regions of a fully-connected
ReLU network with $n_0$ inputs and $L$ hidden layers of widths $n_1, n_2, \cdots, n_L$. Since this work, deriving the lower and upper bounds of the maximum number of linear regions becomes the main research direction \cite{montufar2014number, telgarsky2015representation, montufar2017notes,serra2018bounding,croce2019provable, hu2018nearly,xiong2020number}. All these bounds suggest the expressive ability of depth.
The second interesting example is the finding of the high-capacity-low-reality phenomenon \cite{hu2021model,hanin2019deep}, that the theoretical tight upper bound for the number of polytopes is much larger than what is actually learned by a network, \textit{i.e.}, deep ReLU networks have surprisingly few polytopes both at initialization and throughout the training. This counter-intuitive phenomenon can also be regarded as an implicit bias, which to some extent suggests a deep network does not overfit, since it tends to learn a simple solution.




We observe that current studies on polytopes suffer a critical limit. Either theoretical or empirical studies only stay at the level of counting the number of polytopes, which is far from a complete characterization to ReLU networks. As we know, in a feed-forward network of $L$ hidden layers, each polytope is encompassed by a group of hyperplanes, and each hyperplane is associated with a neuron. The details of how polytopes are formed in a ReLU network are in Supplementary Material. Hence, any polytope is generated by at most $\sum_{i=1}^L n_i$ and at least $n_0+1$ hyperplanes, which is quite a large range. Thus, the face numbers of polytopes can vary a lot. Unfortunately, the existing ``counting'' studies did not accommodate the differences among polytopes. Can we upgrade the characterization of polytopes beyond counting to capture a more complete picture of a neural network? 



To answer this question, in this manuscript, we propose to move one step further to study the shape of polytopes by their number of faces. 1) First, we provide specific constructions for ReLU networks that partition the space into complex polytopes in terms of either the maximum number of faces or the average number of faces. In other words, polytopes can be complicated in the extremal case. 2) Then, we observe that polytopes formed by ReLU networks are surprisingly simple under both initialization and gradient descent, which is a fundamental characteristic of a ReLU network. Here, simplicity means that although theoretically quite diverse and complicated polytopes can be derived, deep networks tend to find a function with many simple polytopes. Our results concretely reveal what simple functions a network learns and its space partition property, which can be regarded as a novel implicit simplicity bias, subjected to the geometric constraint in space partition of ReLU networks. Here, we generalize the concept of implicit bias, which can be intrinsic and not necessarily dependent on any training procedure. 3) We establish a theorem via non-trivial combinatorial techniques to bound the average face numbers of polytopes to a small number. This theorem explains why depth does not make polytopes more complicated. The key idea is that as the depth increases, a ReLU network divides the space into many local polytopes. But to make local polytopes more complex, two or more hyperplanes associated with neurons in succeeding layers should intersect within the given local polytope, which is hard because the area of polytopes is typically small. In brief, our contributions are threefold. 

    {\footnotesize $\bullet$} We point out the limitation of counting \#polytopes. To deepen our understanding of how a ReLU network partitions the space, we propose to investigate the shape of polytopes with the number of faces a polytope has. Investigating polytopes of a network can lead to a more complete characterization of ReLU networks.
    
    {\footnotesize $\bullet$} We first construct ReLU networks that have complex polytopes in the mean or maximal sense. Then, we empirically find that a ReLU network has surprisingly simple polytopes under both initialization and gradient descent. Such an interesting finding is a new kind of implicit bias from the perspective of shapes of linear regions and independent of neural network training procedures. Previously, \cite{hanin2019deep} showed that deep ReLU networks have few polytopes. Our discovery is that polytopes are simple, which is more fine-grained. \textit{Our result and \cite{hanin2019deep} address two essentially different aspects: quantity and shape.}
Compared to \cite{hanin2019deep}, ours more convincingly illustrates a deep network learns a simple function. Showing the number of polytopes is few is insufficient to claim that a network learns a simple solution because a network can have bizarrely complicated polytopes. 

{\footnotesize $\bullet$} We use combinatorial techniques to derive a tight upper bound for the average face number of polytopes under mild conditions, which not only offers a theoretical guarantee to our empirical finding but also explains why depth does not make polytopes more complicated. Many deep learning theories assume infinite width, which essentially delineates the behaviors of a network when it goes wide. Our theory is valuable in characterizing the impact of depth on a network.

\section{Related Work}

\textbf{Studies on polytopes of a neural network.} Besides the aforementioned works \cite{pascanu2013number,  xiong2020number, montufar2014number, hu2018nearly} that count the number of polytopes,
there are increasingly many studies on polytopes of neural networks. \cite{chu2018exact,hanin2019deep,balestriero2020mad} showed that polytopes generated by a network are convex. \cite{zhang2020empirical} studied how different optimization techniques influence the 
local properties of polytopes, such as the inspheres, the directions of the corresponding hyperplanes, and the relevance of the surrounding regions. \cite{gamba2020hyperplane} showed that the angles between
activation hyperplanes defined by convolutional layers are prone to be similar after training. \cite{hu2020measuring} studied the network using an arbitrary activation function. They first used a piecewise linear function to approximate the given activation function. Then, they monitored the change of \#polytopes to probe if the network overfits. \cite{park2021unsupervised} proposed neural activation coding that maximizes the number of linear regions to enhance the model's performance. \cite{humayun2024deep} computed the density of linear regions as the measure of the local complexity to investigate the phenomenon where generalization occurs long after a network achieves near-zero training error. \cite{khalife2023neural} and \cite{hertrich2021towards} used polyhedral methods to investigate the upper and lower bounds on the sizes of the neural networks required to represent the class of piecewise functions. Our work goes beyond counting the number of polytopes to consider the shapes of polytopes, with the goal of delineating a more complete picture of neural networks. \cite{humayun2023splinecam} exactly computed the geometry of a deep ReLU network's mapping such as decision boundaries and then proposed the SplineCAM to attribute the importance of features for network interpretability.



\textbf{Implicit bias of deep learning.} A network used in practice
is highly over-parameterized compared to the number of training samples. A natural question is often asked: why do deep networks not overfit? To address this question, extensive studies have proposed that a network is implicitly regularized to learn a simple solution. Implicit regularization is also referred to as an implicit bias. Gradient descent algorithms are widely believed to play an essential role in capacity control even when it is not specified in the loss function \cite{gunasekar2018characterizing, soudry2018implicit, arora2019implicit, sekhari2021sgd, lyu2021gradient}. \cite{du2018gradient, woodworth2020kernel} showed that the optimization trajectory of neural networks stays close to the initialization with the help of neural tangent kernel theory. A line of works \cite{arora2019fine, cao2019towards, yang2019fine, choraria2022spectral} have analyzed the bias of a deep network towards lower frequencies, which is referred to as the spectral bias. It was shown in \cite{arora2018stronger,yu2017compressing} that replacing weight matrices with low-rank matrices only deteriorates a network's accuracy very moderately. \cite{ongie2022role,letraining} identified the low-rank bias in linear layers of neural networks with gradient flow. Both theoretical derivation \cite{tu2016low, li2020towards} and empirical findings \cite{jing2020implicit,huh2021low,galanti2023sgd} suggested that gradient descent tends to find a low-rank solution. What's more, weight decay is a necessary condition to achieve the low-rank bias \cite{galanti2023sgd}. 

In contrast, our investigation identifies a new implicit bias from the perspective of linear regions, which draws two highlights: 1) The core of the implicit bias is to emphasize that a network is implicitly regularized to lead to a simple solution. Therefore, such an implicit regularization is not necessarily due to training procedures. What we discover is an intrinsic regularization from the geometric space partition of ReLU network. 2) Different from most implicit biases highlighting a certain property of a network, our implicit bias straightforwardly reveals what kind of simple functions a network learns. Our finding is relevant to the spectral bias. 
Since polytopes are both few and simple, a ReLU network does not produce a lot of oscillations in all directions, which roughly corresponds to a low-frequency solution.


\section{Preliminaries}
\label{sec:pre}

Throughout this paper, we always assume that the input space of an NN is a $d$-dimensional hypercube $C(d,B):= [-B,B]^d = \{\x=(x_1,x_2,\ldots,x_d)\in \mathbb{R}^d: -B\leq x_i \leq B \}$ for some large enough constant~$B$. Furthermore, we need the following definition for linear regions (polytopes).  

\begin{definition}
[Linear regions (polytopes) \cite{hanin2019complexity}]\label{def:activation-regions}
\textit{
Suppose that $\mN$ is a ReLU NN with $L$ hidden layers and input dimension $d$. 
An activation pattern of $\mN$ is a function $\mP$ from the set of neurons to the set $\{1,-1\}$, i.e., for each neuron $z$ in $\mN$, we have $\mP(z) \in \{1,-1\}$.   
Let $\theta$ be a fixed set of parameters in $\mN$, and $\mP$ be an activation pattern. Then the region corresponding to $\mP$ and $\theta$ is
$
\mR(\mP;\theta) :=
\{
X\in C(d,B):   z(X;\theta)\cdot {\mP(z)} >0
\},
$
where $z(X;\theta)$ is the pre-activation of a neuron $z$ in $\mN$. A linear region of $\mN$ at $\theta$ is a non-empty set $\mR(\mP,\theta)\neq \emptyset$ for some activation pattern $\mP$. Let $R_{\mN,\theta}$ be the number of linear regions of $\mN$ at $\theta$, i.e., 
$
R_{\mN,\theta} := \#\{  \mR(\mP;\theta): \mR(\mP;\theta)\neq \emptyset ~ \text{ for some activation pattern } \mP    \}.
$
Moreover, let $R_{\mN}:=\max_\theta R_{\mN,\theta}$ denote the maximum number of linear regions of~$\mN$ when $\theta$ ranges over $\mathbb{R}^{\#weights+\#bias}$.
}
\end{definition}

In the following, Preliminary \ref{pre:convex_poly} shows that polytopes generated by a ReLU network are convex. The detailed explanation of Preliminary \ref{pre:convex_poly} can be seen in Appendix A. Preliminary \ref{polytope_denotation} introduces how to denote a polytope by the linear functions associated with hyperplanes of the polytope. Preliminary \ref{hit_and_run} introduces the hit-and-run algorithm, a representative algorithm to count the \#faces of a polytope.

\begin{pre}[Polytopes of a neural network]
A neural network with ReLU activation partitions the input space into many polytopes (linear regions), such that the function represented by this neural network becomes linear when restricted in each polytope (linear region). Each polytope corresponds to a collection of activation states of all neurons, and each polytope is convex \cite{chu2018exact}. In this paper, we mainly focus on $(n_0-1)$-dim faces of a $n_0$-dim polytope. \textbf{For convenience, we just use the terminology face to represent an $(n_0-1)$-dim facet of an $n_0$-dim polytope.} 
\label{pre:convex_poly}
\end{pre}

\begin{pre}[Simplex and simplicial complex]
\label{def:simplex}
A \textbf{simplex} is just a generalization of the notion of triangles or tetrahedrons to any dimensions. More precisely, a $D$-simplex $S$ is a $D$-dimensional convex hull provided by convex combinations of  $D+1$ affinely independent vectors $\{\v_i\}_{i=0}^D \subset \mathbb{R}^D$. In other words, $\displaystyle S = \left\{ \sum_{i=0}^D \xi_i \v_i ~|~ \xi_i \geq 0, \sum_{i=0}^D \xi_i = 1 \right \}$.  
The convex hull of any subset of $\{\v_i\}_{i=0}^D$ is called a 
face of $S$. A \textbf{simplicial complex} $\displaystyle \mathcal{S} = \bigcup_\alpha S_\alpha$ is composed of a set of simplices $\{S_\alpha\}$ satisfying: 1) every face of a simplex from $\S$ is also in $\S$; 2) the non-empty intersection of any two simplices $\displaystyle S_{1},S_{2}\in \S$ is a face of both $S_1$ and $S_2$. 
A \textbf{triangulation of a polytope} $P$ is a partition of $P$ into simplices such that the union of all simplices equals $P$, and the intersection of any two simplices is a common face or empty. The triangulation of a polytope results in a simplicial complex.
\end{pre}

\begin{pre}[Complexity of the Shape of a Polytope]
We use the number of faces (\#faces) a polytope has to measure the complexity of its shape. We also count the number of highest dimensional simplices (\#simplices) a polytope encompasses as the intermediate results to bound the \#faces. \textbf{The maximum \#faces a polytope has is the total number of neurons ($\Gamma$) of a network. Unofficially, we define the simplicity threshold as $\Gamma/2$. If a polytope has \#faces smaller than $\Gamma/2$, it is deemed simple; otherwise, it is complicated.}   
\end{pre}

\begin{pre}[Denote a polytope by its hyperplanes]
A hyperplane in $\mathbb{R}^{d}$ is associated with a linear function $h(\x)$. We write $h^{+}=\left\{ \x\in\mathbb{R}^{d}:h(\x)\geq 0 \right\}$ and $h^{-}=\left\{ \x\in\mathbb{R}^{d}:h(\x)< 0 \right\}$. A region formed by $n$ hyperplanes $h_1,\ldots,h_n$ can be denoted as $\cap_{i=1}^{n}h_i^{\chi_i}$, $\chi_{i}\in\left\{+,-\right\}$. 
\label{polytope_denotation}
\end{pre}

\begin{pre}[Hit-and-run algorithm that counts \#faces]
In Supplementary Materials, we show that the linear region where a given input $\x$ lies corresponds to a group of inequalities determined by the activation states of all neurons. Mathematically, a polytope with the dimension $n_0$ is defined as $\{\x\in\mathbb{R}^{n_0}~|~\mathbf{a}_k \x^\top+b_k\leq 0, k \in [K]\}$. Each inequality corresponds to a hyperplane. However, not all hyperplanes are faces of the encompassed polytope. \textit{We call the inequalities that are faces of the encompassed polytope non-redundant inequalities.} \textbf{Thus, counting the \#faces of polytopes is equivalent to counting the number of non-redundant inequalities.} Although for high dimensional problems, it's difficult to find all the non-redundant inequalities in a short time, we can, however, apply probabilistic methods to find as many as possible to provide an effective estimation. There are various probabilistic methods to find necessary linear inequalities \cite{caron1997redundancy}. The hit-and-run algorithms is a representative Monte Carlo sampling method As Figure \ref{Figure_hit_and_run} shows, their basic idea is to randomly "hit" the boundaries of the polytope from its interior point. The interior point is exactly the given input $\x$.
\label{hit_and_run}
\end{pre}



\begin{figure}[htb!]
\vspace{-0.4cm}
\center{\includegraphics[width=0.5\linewidth] {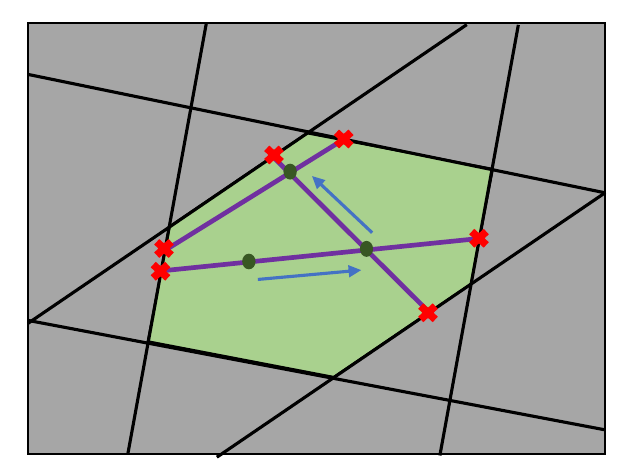}}
\caption{The Hit-and-Run algorithm that detects the faces of polytopes and counts them. }
\label{Figure_hit_and_run}
\vspace{-0.3cm}
\end{figure}

\section{Construction of Complicated Polytopes}
\label{sec:construction}

To form a clear basis of comparison for the polytopes being simpler, now we purposely design two networks as a representative example that partitions the space into many very complicated polytopes in the sense of either the average number of faces or the maximum number of faces. 

{\footnotesize $\bullet$} \textbf{The average \#faces.} Our core idea is to show that there exist parameters in a neuron to constrain the hyperplane into a specific domain such that the new neuron only creates another at least equally-complicated polytope inside a complicated polytope, without partitioning a region outside a complicated polytope. Then, in the average sense, polytopes are complex.

Let us use a two-dimensional example as shown in Figure \ref{Figure_counterexample}(a) to illustrate our idea. Suppose there are three neurons in the first hidden layer denoted as 
\begin{equation}
    z_i=\sigma(h_i(\x)) = \sigma(p_1^{(i)} x_1 +p_2^{(i)} x_2 +r^{(i)}), i=1,2,3.
\end{equation}
Without loss of generality, we assume three lines formed by these three neurons constitute a triangle, and the central triangular region $\Omega=h_1^{-}\cap h_2^{-}\cap h_3^{+}$, which means that only the third neuron is activated in $\Omega$. We prescribe the neuron in the second hidden layer computes 
\begin{equation}
    y(\x)=\sigma(-\mu_1 z_1-\mu_2 z_2-z_3+c),
\end{equation}
where $\mu_1,\mu_2>0$, and $c>0$. Let us see how $y(\x)$ cuts the space: 1) $y(\x)$ splits the regions $h_1^{+}\cap h_2^{-}\cap h_3^{+}$ and $h_1^{-}\cap h_2^{+}\cap h_3^{+}$ into two regions. However, as $\mu_1$ and $\mu_2$ increase, the blue and orange regions will become smaller. In the infinity limit, $y(\x)$ does not partition regions $h_1^{+}\cap h_2^{-}\cap h_3^{+}$ and $h_1^{-}\cap h_2^{+}\cap h_3^{+}$; $y(\x)$ divides $\Omega$ into two equally complicated polytopes. Thus, in terms of the average number of faces, polytope partitioning is complex. 

Now, let us formally provide our two-hidden-layer construction for $\mathbb{R}^d$:
\begin{equation}
    \begin{cases}
        & z_i=\sigma(h_i(\x)) = \sigma\left(\sum_{j=1}^d p_j^{(i)} x_j+r^{(i)}\right), i=1,\cdots,n  \\
        & y_i = \sigma(-\sum_{j=1}^{d-1}\mu_j^{(i)} z_j - z_d+b), i=1,\cdots,m \\
        & \mathrm{output} = y_1+y_2+\cdots+y_m,
    \end{cases}
\end{equation}
where we let all hyperplanes of $n$ neurons in the first hidden layer intersect at one vertex to form a cone with $n$ faces. Since the second layer only cuts this cone and does not generate extra polytopes outside the cone when $\mu_j, j\to \infty$, the average face number of polytopes is $\frac{n\cdot m+c_1}{c_2+m} \approx n$, when $m$ goes large, where $c_1$ is the number of faces and $c_2$ is the number of polytopes except for the $n$-face polytope. 

\begin{figure}[htb!]
\vspace{-0.2cm}
\center{\includegraphics[width=\linewidth]{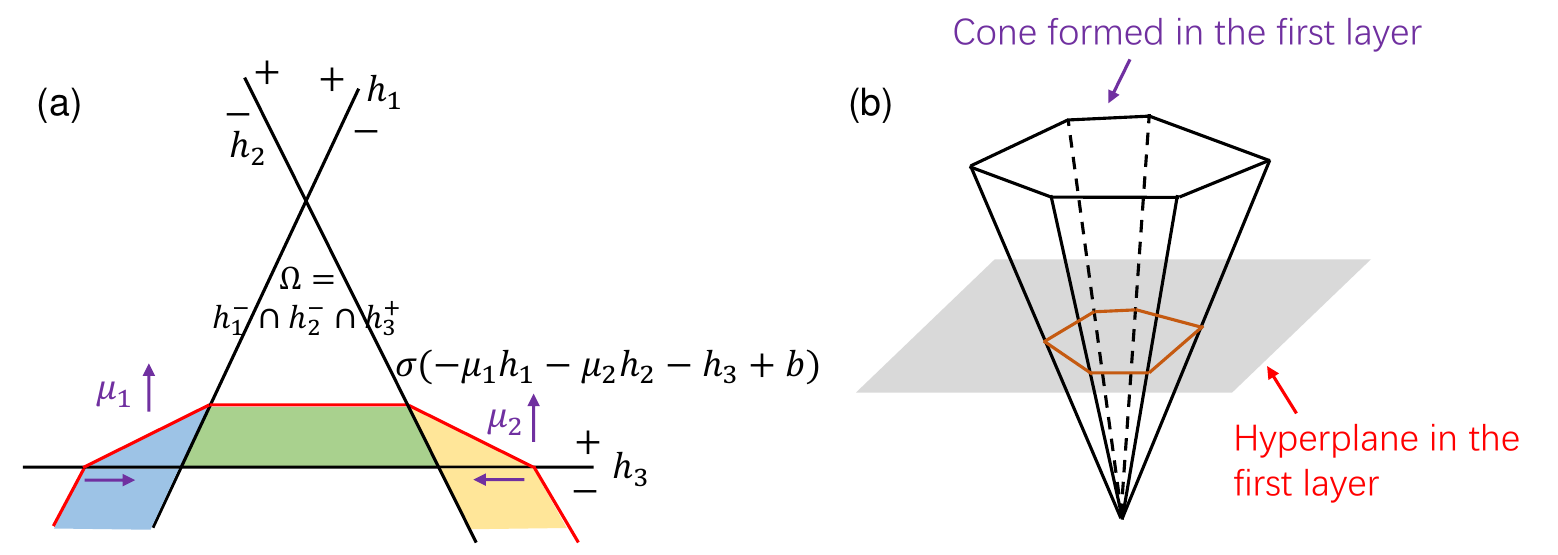}}
\caption{An explanatory graph of how to construct a network that partitions the space into complicated polytope in the average sense. A cone is generated by the first hidden layer, and neurons in the second hidden layer keep cutting the cone without cutting the regions outside the cone.}
\label{Figure_counterexample}
\vspace{-0.15cm}
\end{figure}
Notably, we can stack more layers whose neurons keep cutting the $\Omega$. Thus, this construction can be easily extended to an arbitrarily deep network.

{\footnotesize $\bullet$} \textbf{The maximum \#faces.} For a fully-connected network, the maximum \#faces a polytope can have is the number of neurons. To achieve this maximum, we need to show that there exists a parameter configuration that can make all neurons contribute to one polytope. 

Let us use a two-dimensional example as shown in Figure \ref{Figure_maximum} to illustrate our idea. Suppose there are three neurons in the first hidden layer denoted as 
\begin{equation}
    z_i=\sigma(h_i(\x)) = \sigma(p_1^{(i)} x_1 +p_2^{(i)} x_2 +r^{(i)}), i=1,\cdots,5.
\end{equation}
Without loss of generality, we assume five lines formed by these five neurons constitute a triangle, and the central triangular region $\Omega=h_1^{+}\cap h_2^{+}\cap h_3^{+}\cap h_4^{+}\cap h_5^{+}$, which means that all neurons are activated in $\Omega$. Next, as Figure \ref{Figure_maximum} shows, we select two points in the neighboring faces, respectively, to determine a face such that a neuron in the second hidden layer exactly forms this face. Suppose that this face is 
\begin{equation}
    t_1 x+t_2 y+s=0
\end{equation}
Let the neuron in the second hidden layer only take $z_3$ and $z_5$. 
\begin{equation}
    y(\x)=\sigma(\alpha z_3+\beta z_5+\gamma),
\end{equation}
where $\alpha, \beta, \gamma$ fulfill that
\begin{equation}
\begin{bmatrix}
p_1^{(3)} & p_1^{(5)} & 0\\
p_2^{(3)} & p_2^{(5)} & 0  \\
r^{(3)} & r^{(5)}& 1
\end{bmatrix} \begin{bmatrix}
\alpha\\
\beta\\
\gamma 
\end{bmatrix} = \begin{bmatrix}
t_1\\
t_2\\
s
\end{bmatrix}.
\end{equation}
When selecting a new pair of points, one can easily ensure that the number of faces must increase. Such a technique can generalize to high-dimensional inputs and arbitrary depth.  

\begin{figure}[htb!]
\vspace{-0.2cm}
\center{\includegraphics[width=\linewidth]{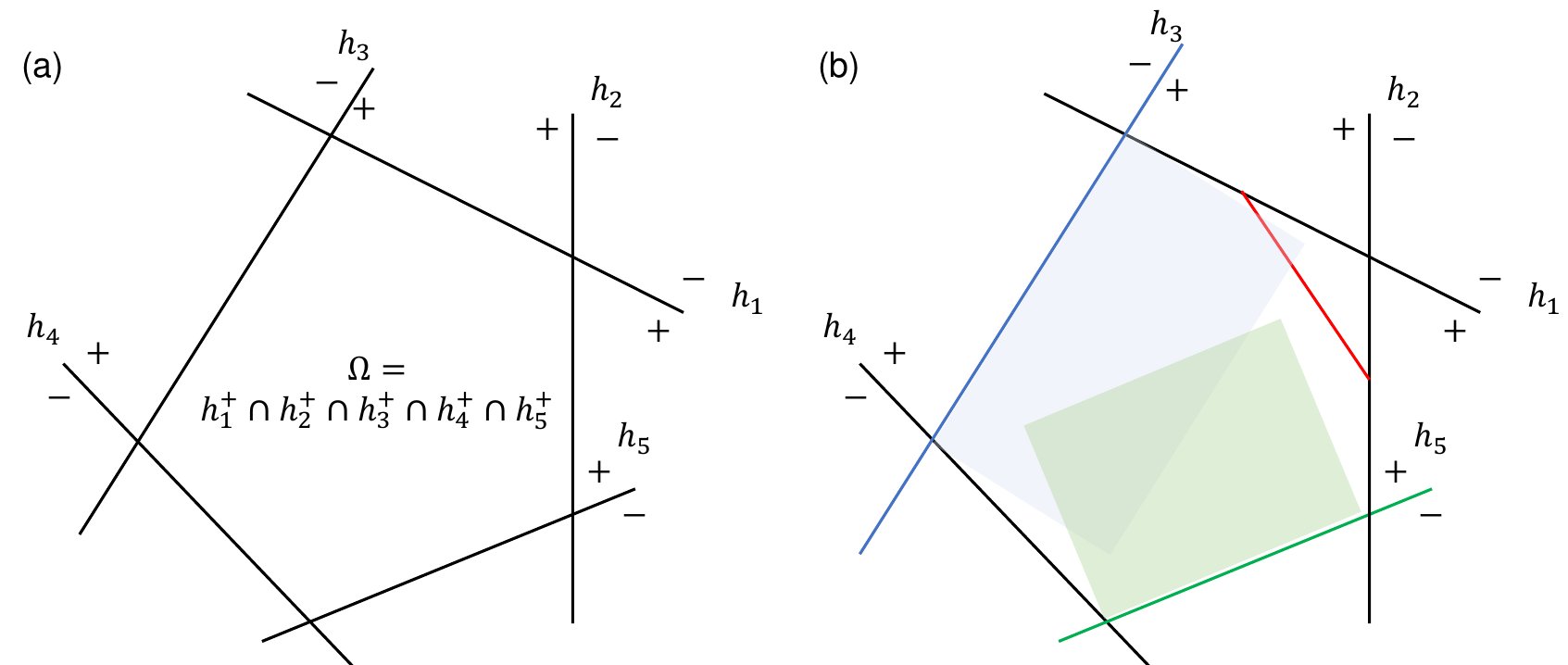}}
\caption{An explanatory graph of how to construct a network that partitions the space into complicated polytopes in the maximal sense. There exists a parameter configuration that can result in a polytope whose number of faces equals to the number of neurons in a network.}
\label{Figure_maximum}
\vspace{-0.35cm}
\end{figure}

\section{Deep ReLU Networks Have Simple Polytopes}
\label{sec:experiments}


In the last section, it can be seen that we need to tune neurons' weights to fulfill harsh conditions like very large weights and linear constraints such that a ReLU network can divide the space into complicated polytopes. However, a normally-trained ReLU network should not behave that way.

Along this line, by analyzing \#faces a polytope contains, we empirically observe that linear regions formed by ReLU networks are much simpler than the worst case under both initialization and gradient descent, which is a high-capacity-low-reality phenomenon and a new implicit bias, suggesting what simple solutions a deep network learns. We validate our findings comprehensively and consistently at different initialization methods, network depths, sizes of the outer bounding box, and biases. Furthermore, we showcase that during the training, although the number of linear regions increases, linear regions keep their simplicity. Lastly, our experiments are not only on low-dimensional inputs but also extended to high-dimensional inputs by Monte Carlo simulation.

\subsection{Initialization}

We validate four popular initialization methods: Xavier uniform, Xavier normal\footnote{\url{https://pytorch.org/docs/stable/nn.init.html}}, Kaiming, orthogonal initialization \cite{he2015delving}. For each initialization method, we use two different network architectures (3-40-20-1, 3-80-40-1). The bias values are set to 0.01 for all neurons. A total of 8,000 points are uniformly sampled from $[-1,1]^3$ to compute the polytope. At the same time, we check the activation states of all neurons to avoid counting some polytopes more than once. Each experiment is repeated five times. 

{\footnotesize $\bullet$} \textbf{Initialization methods}: Figure \ref{Figure_initialization_method} shows the histogram of the \#simplices each polytope has under different initialization methods. Hereafter, if no special specification, the x-axis of all figures denotes the number of faces a polytope has, and the y-axis denotes the count of polytopes with a certain number of faces.
The spotlight is that for all initialization methods and network structures, all polytopes are simple compared to the extreme they can reach. Moreover, comparing the network structure $3-80-40-1$ and $3-40-20-1$, it is observed that the number of faces polytopes have does not increase. The achieved polytope is far simpler than the theoretically most complicated polytope, which is $120$.

\begin{figure}[htb!]
\vspace{-0.3cm}
\center{\includegraphics[width=\linewidth]{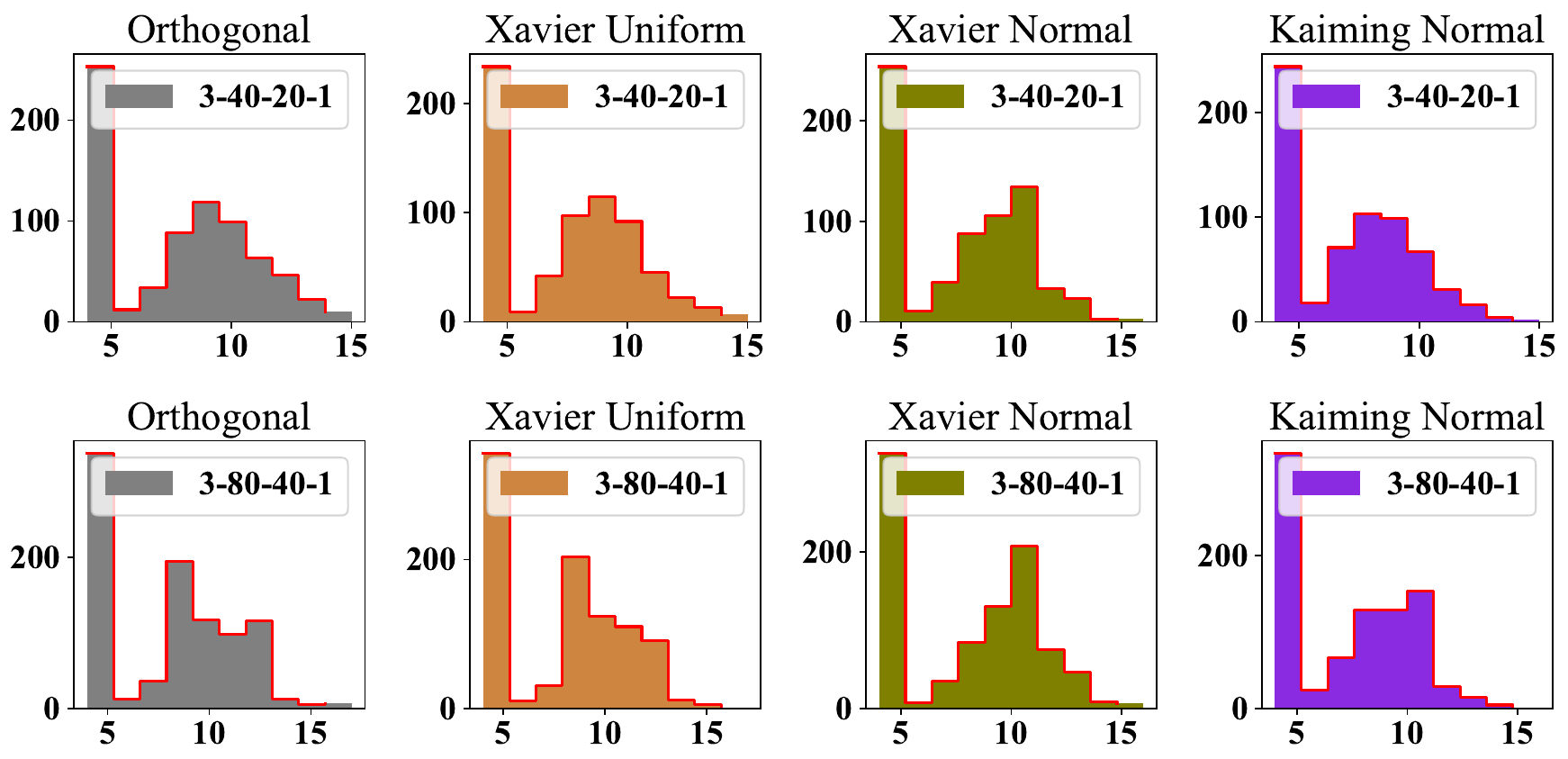}}
\caption{Deep ReLU networks have simple linear regions at different initialization methods. }
\label{Figure_initialization_method}
\vspace{-0.2cm}
\end{figure}

{\footnotesize $\bullet$} \textbf{Depths}: Here, we evaluate if the simplicity of polytopes still holds for deeper networks. This question is non-trivial, since a deeper network can theoretically generate more complicated polytopes. Will the depth break the simplicity? We choose four different widths (20, 40, 80, 160). For comprehensiveness, the network initialization methods are the Xavier uniform, Xavier normal, Kaiming, and orthogonal initialization. The depth is set to 5 and 8, respectively. The bias value is 0.01. Likewise, a total of 8,000 points are uniformly sampled from $[-1,1]^3$ to compute the polytope. At the same time, we check the activation states of all neurons to avoid counting some polytopes more than once. Each experiment is repeated five times. The results under the Xavier uniform initialization are shown in Figure \ref{Figure_depth_xu}, from which we draw three highlights. First, we find that both going deep and going wide can increase the number of polytopes at different initializations. But the effect of going deep is much more significant than that of going wide. Second, when the network goes deep, although the total number of polytopes increases, simple polytopes still dominate among all polytopes. Third, for different initialization methods and different depths, the dominating polytope is slightly different. For example, the dominating polytopes for the network 3-40-40-40-40-40-1 under Xavier normal initialization are those with 4$\sim$10 faces, far smaller than the specific constructions provided in the last subsection.

\begin{figure}[htb!]
\vspace{-0.3cm}
\center{\includegraphics[width=\linewidth] {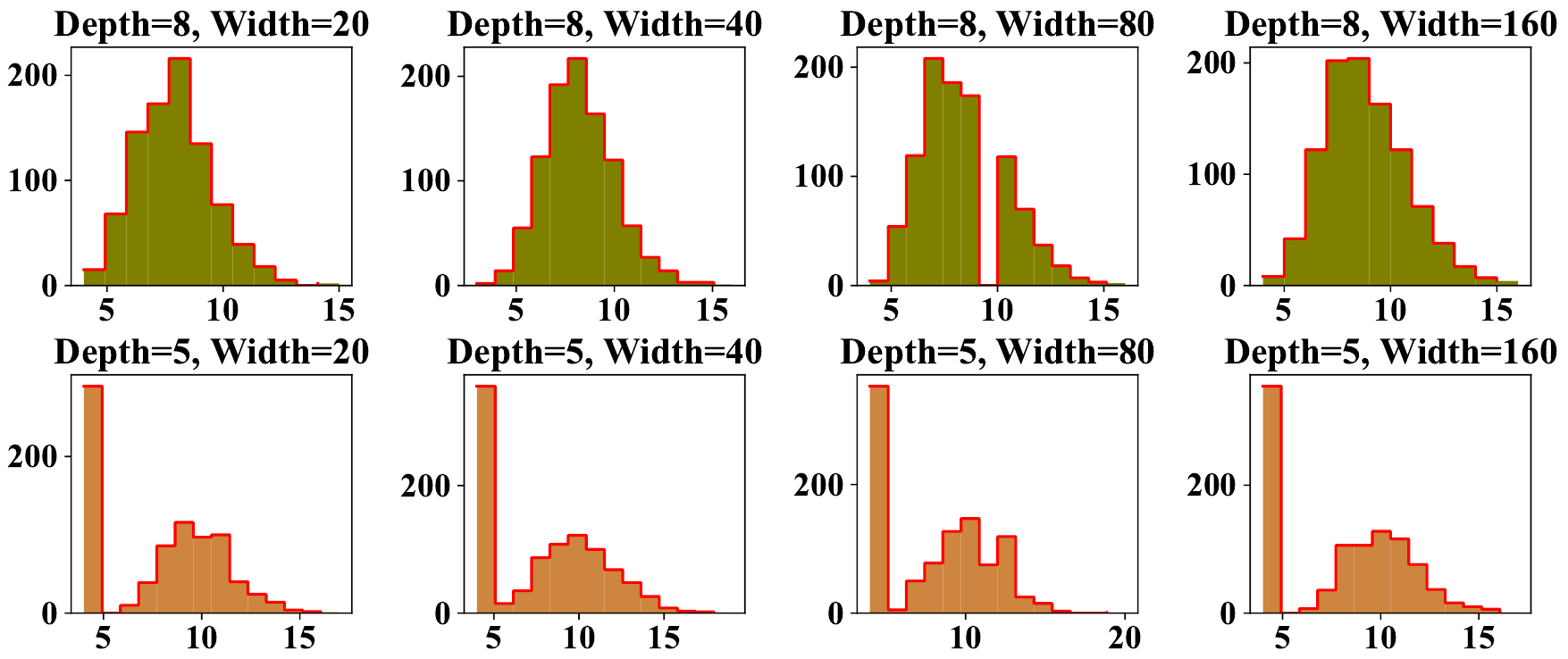}}
\caption{The simplicity holds true for deep networks. }
\label{Figure_depth_xu}
\vspace{-0.1cm}
\end{figure}

{\footnotesize $\bullet$} \textbf{Biases}: Here, we are curious about how the bias value of neurons will affect the distribution of polytopes. To address this issue, we set the bias values to $0, 0.01, 0.05, 0.1$, respectively for the network 3-80-40-1. The outer bounding box is $[-1,1]^3$. A total of 8,000 points are uniformly sampled from $[-1,1]^3$ to compute the polytope. At the same time, we check the activation states of all neurons to avoid counting some polytopes more than once. Each experiment is repeated five times. The initialization methods are the Xavier uniform, Xavier normal, Kaiming, and orthogonal initialization. Figure \ref{Figure_bias_xavier} is from the Xavier uniform. We observe that as the bias value increases, more polytopes are produced. However, the number of simple polytopes still takes up the majority. It is worthwhile mentioning that when the bias equals 0, the simplicity is clear. The bias=0 is the extremal case, where all hyperplanes of the first layer intersect at the origin, and much fewer faces in polytopes are generated.

\begin{figure}[htb!]
\vspace{-0.3cm}
\center{\includegraphics[width=\linewidth] {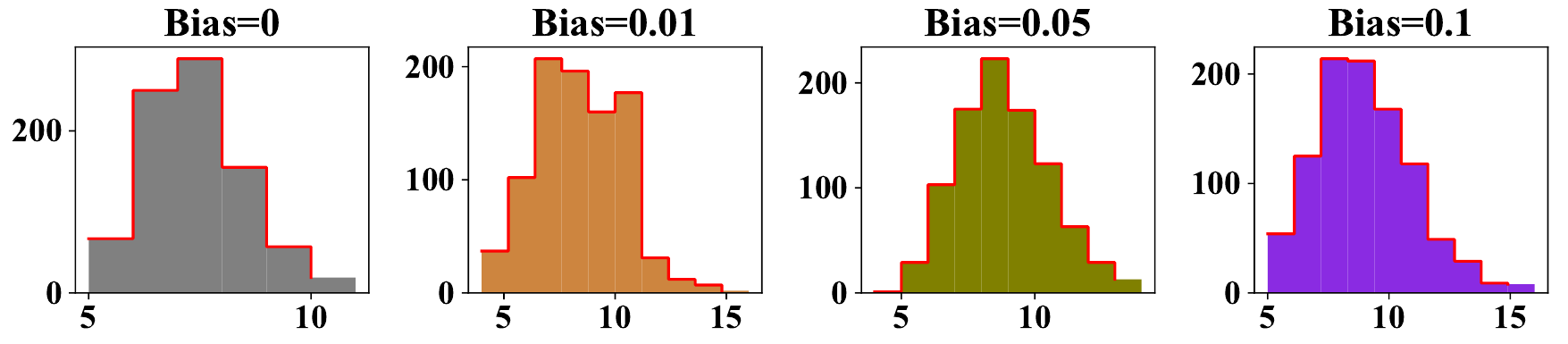}}
\caption{The simplicity holds true for different bias values under the orthogonal initialization. }
\label{Figure_bias_xavier}
\vspace{-0.35cm}
\end{figure}


\vspace{-0.3cm}
\subsection{Training}
\label{sec:dynamics}

Earlier, we show that at the initialization stage, deep networks exhibit simple linear regions. It is natural to ask \textit{will the simplicity of linear regions be broken during training}? We answer this question by training a fully-connected network using ReLU activation function on a real-world problem and counting the simplices of each polytope. The task is to predict if a COVID-19 patient will be at high risk, given one's health status, living habits, and medical history. This prediction task has 388,878 raw samples, and each has 5 medical features including `HIPERTENSION',`CARDIOVASCULAR', `OBESITY', `RENAL CHRONIC', `TOBACCO'. The labels are `at risk' or `no'. The detailed descriptions of data and this task can be referred to in Kaggle\footnote{\url{https://www.kaggle.com/code/meirnizri/covid-19-risk-prediction}}. The data are preprocessed as follows: The discrete value is assigned to different attributes. If a patient has that pre-existing disease or habit, 1 will be assigned; otherwise, 0 will be assigned. Then, the data are randomly split into training and testing sets with a ratio of 0.8:0.2. We implement a network of 5-20-20-1. The optimizer is Adam with a learning rate of 0.1. The network is initialized by Xavier uniform. The loss function is the binary cross-entropy function. The epoch number is 400 to guarantee convergence. A total of 8,000 points are uniformly sampled from $[-1,1]^3$ to compute the polytope. The outer bounding box is $[-5,5]^3$ to ensure as many polytopes as possible are counted.

\begin{figure}[htb!]
\center{\includegraphics[width=\linewidth] {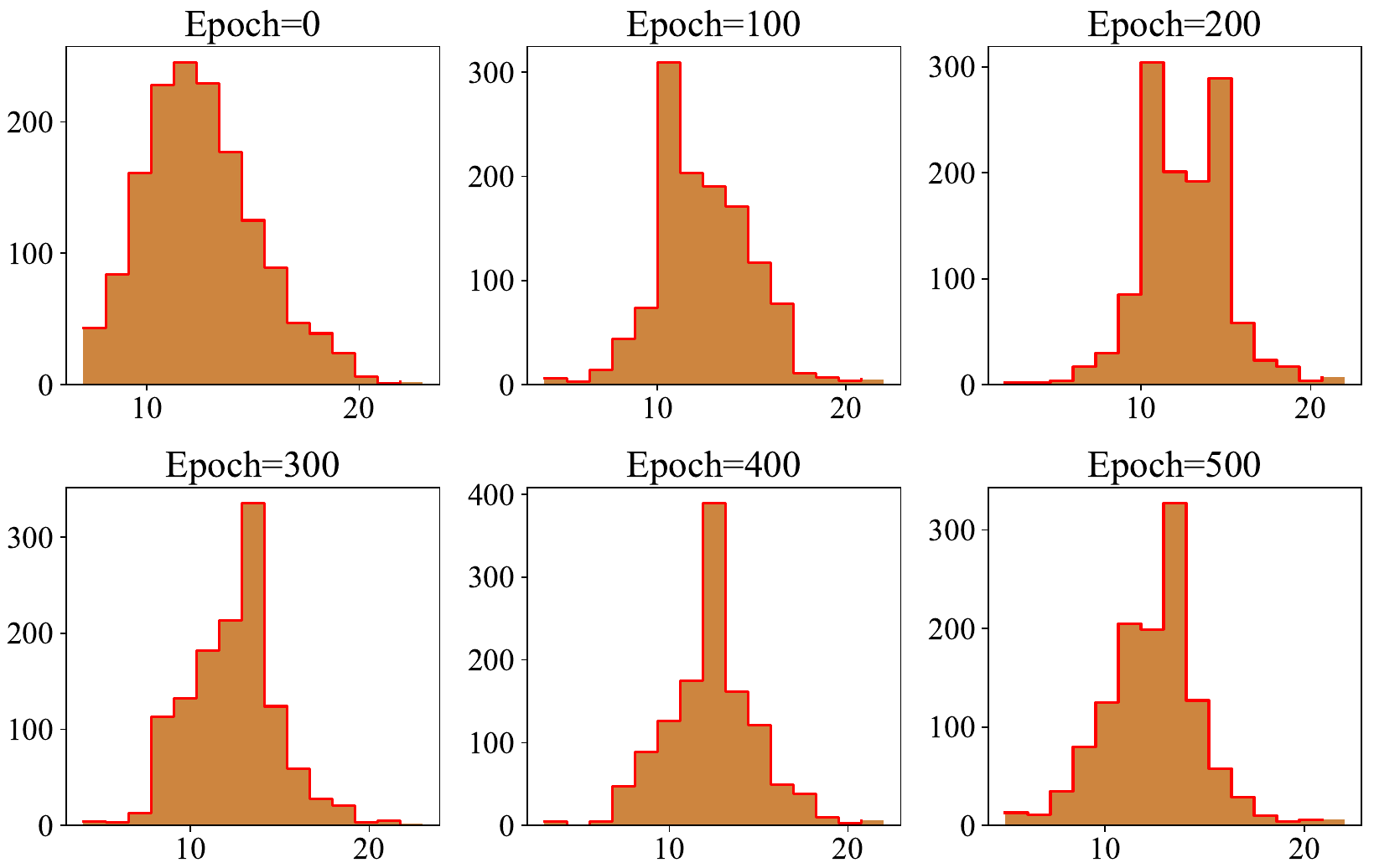}}
\caption{The results over a COVID dataset show that throughout the training, most polytopes are simple, despite that the number of linear regions drops during the training. }
\label{Figure_training}
\vspace{-0.3cm}
\end{figure}

Figure \ref{Figure_training} shows that as the training goes on, the total number of linear regions drops compared to the random initialization. It is observed that the number of polytopes with 10-13 faces goes up, and the number of polytopes with fewer than 8 faces goes down. It suggests that the network may be primarily using them to fit data. Meanwhile, it means polytopes generated by the network are slightly tending towards complexity as the training proceeds. However, after the training ends, the most complicated polytopes still have no more than 22 faces, which is approximately half of the number of neurons ($20+20=40$). As a result, we can still conclude that most polytopes are simple.

\subsection{Beyond Small Inputs and Fully-Connected Networks via Monte Carlo Simulation}

To prevent our observation from being biased by i) the input being so small, ii) the width of the hidden layers being so much larger than the input, and iii) networks being fully-connected, we need to \textit{empirically estimate the shape of polytopes for high-dimensional inputs}. Here, we compute the average \#faces produced by LeNet-5 trained on MNIST.

We use the above method to estimate the number of faces of polytopes generated by a modified LeNet-5 trained on the MINST dataset. The modification is removing one convolutional layer and replacing all activation functions with ReLU. We randomly generate 200 instances from a uniform distribution on $[0,1]^{28\times 28}$. For each instance, we iteratively apply the Hit-and-Run process to detect the faces of the polytope, and record every newly found faces. We set a checkpoint every 1000 iterations. Once the algorithm cannot find any new face in the last 1000 iterations, we consider it has found most of the faces of that polytope, and stop the process. The distributions of the number of faces we find and the number of iterations taken are shown in Figure \ref{fig:histogram_appendix}. As can be seen, among $26,796$ inequalities (The maximum number of faces a polytope can have), our algorithm finds $1,677$ faces on average for each polytope. On no polytopes, our algorithm can find more than $2,000$ boundaries before reaching the stopping criteria, which means all polytopes are simple compared to the maximum. Therefore, this result shows that, compared with the complex structure of the network, its polytopes indeed have much fewer faces. Figure \ref{fig:histogram_iter} shows the number of iterations it takes to identify the number of faces of each polytope. On average, after around $1.2\times 10^5$ iterations, the algorithm cannot find a new face.  

\begin{figure}[htb!]
    \centering
    \includegraphics[width=0.6\linewidth]{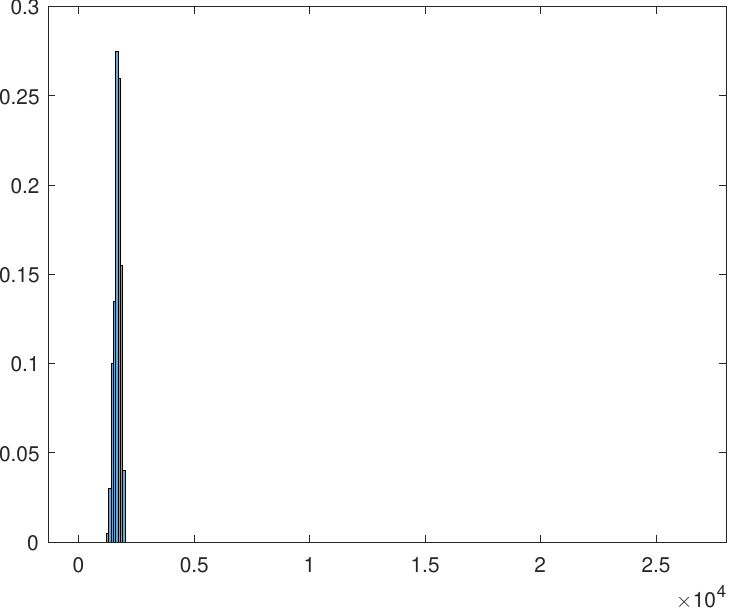}
    \caption{The distribution of the number of faces polytopes have. There are $26,796$ inequalities (The maximum number of faces a polytope can have). However, our algorithm finds that all polytopes have no more than $5,000$ faces. This means that polytopes are simple.}
    \label{fig:histogram_appendix}
\end{figure}

\begin{figure}[htb!]
\vspace{-0.2cm}
    \centering
    \includegraphics[width=0.6\linewidth]{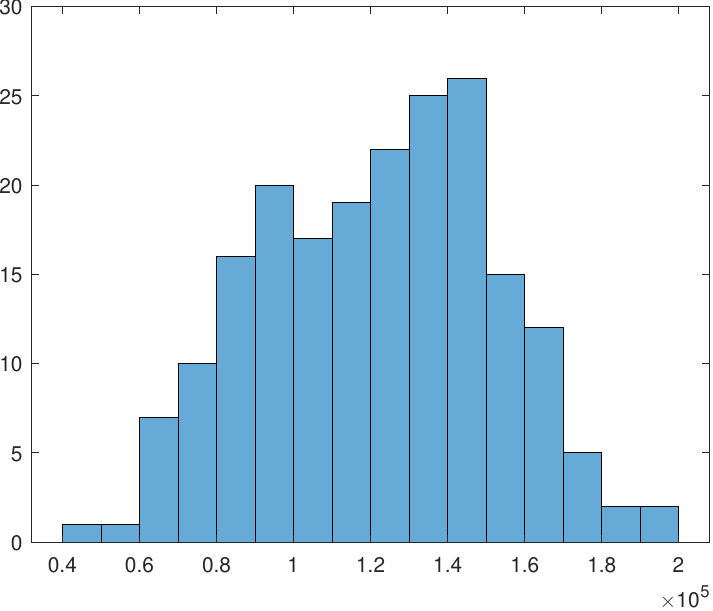}
    \caption{The number of iterations it takes to identify the number of faces of each polytope. On average, after around $1.2\times 10^5$ iterations, the algorithm cannot find a new face.}
    \label{fig:histogram_iter}
    \vspace{-0.2cm}
\end{figure}

\textbf{Visualization.} We also train networks on MNIST, following the same procedure in \cite{hanin2019deep}. Here, we visualize the polytopes in the cross-section plane. We initialize a network of size 784-7-7-6-10 with Kaiming normalization. The batch size is 128. The network is trained with Adam with a learning rate of 0.001. The total epoch number is set to 480, which ensures the convergence of the network.

Figure \ref{Figure_cross_section} shows the cross-section of the function learned by a network at different epochs. A cross-section is a plane that passes through a randomly-selected image $I$ from MNIST along two randomly-selected directions: $\alpha, \beta$. Mathematically, $\mathbf{I}'=\mathbf{I}+a\cdot \mathbf{\alpha} + b\cdot \mathbf{\beta}$, where $a$ and $b$ are scalars. Figure \ref{Figure_cross_section} shows that as the training goes on, the number of polytopes increases. But almost all the polytopes are triangles or quadrilaterals. Our basic assumption is that if the original polytope is complex, its cross-section is also complex. Therefore, we can conclude the simplicity of these polytopes based on the simplicity of cross-sectioned visualization.

\begin{figure}[htb!]
\center{\includegraphics[width=\linewidth] {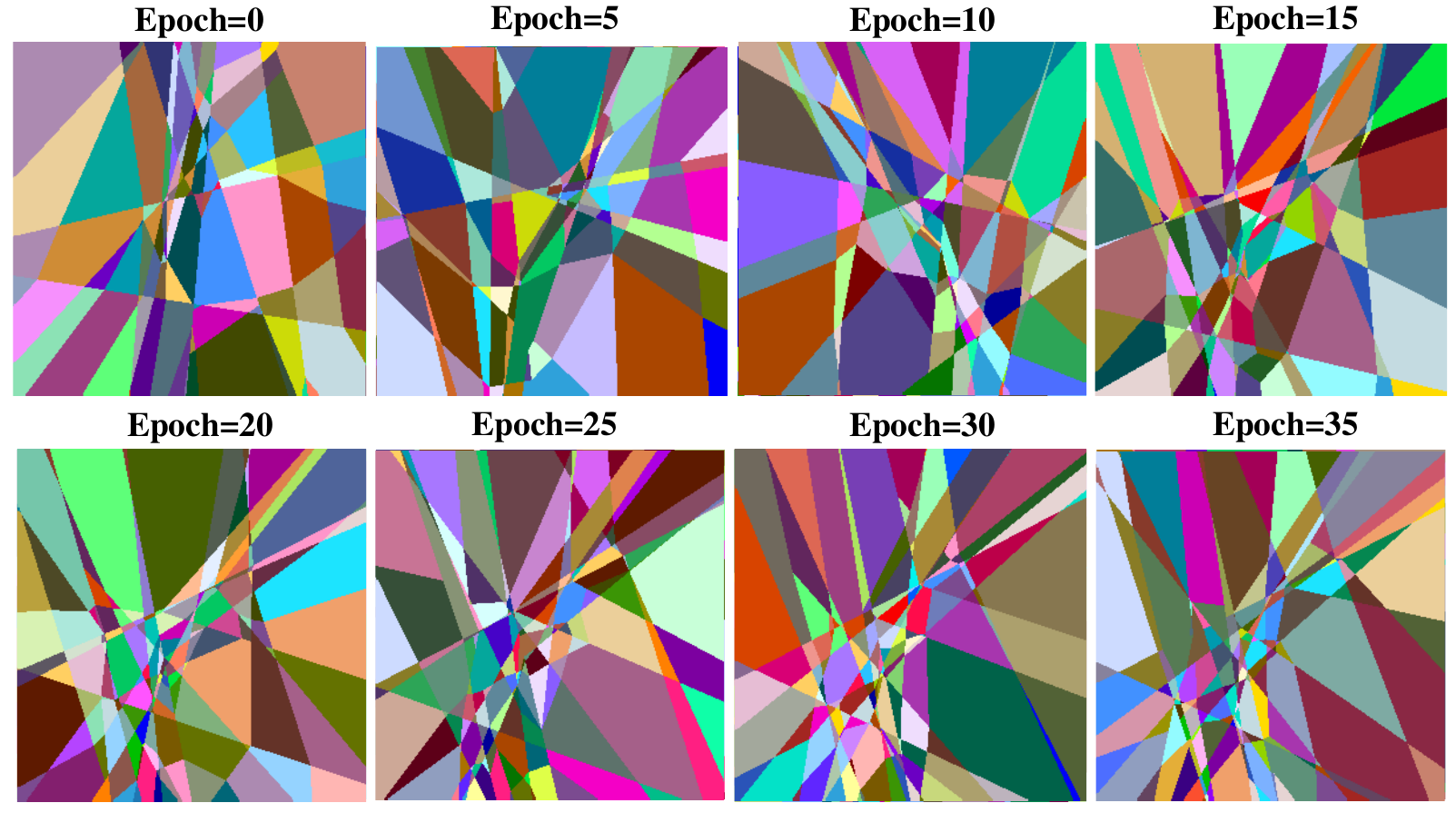}}
\caption{A cross-sectional visualization of the polytopes learned by a network over MNIST at different epochs. Almost all the polytopes are triangles or quadrilaterals.}
\label{Figure_cross_section}
\vspace{-0.5cm}
\end{figure}


\section{Theoretical Explanation}
\label{sec:math}

In this section, we seek to provide a theoretical explanation for the simple polytope phenomenon. We establish a theorem that bounds the average face numbers of polytopes of a network to a small number under some mild assumption, thereby substantiating our finding. Our theoretical derivation is twofold: initialization and after training. 

\textbf{Geometric heuristics of multi-layer networks.}
Generically, we argue that a deep ReLU network should still have simple polytopes. We think that the simplicity of polytopes is given rise to one reason.
Since as the depth increases, a ReLU network divides the space into many local polytopes, to yield a complicated polytope from a local polytope, two or more hyperplanes associated with neurons in the later layers should intersect within the given local polytope, which is hard because the area of polytopes is typically small. As such, the complexity of polytopes probably only increases moderately as the network goes deeper. 


We first estimate the bound of the maximum \#simplices as intermediate results to bound \#faces.

\vspace{-0.4cm}
\subsection{Bound of the Maximum \#Simplices}

\begin{thm}[Upper Bound]\label{thm:upper}
Let $\mN$ be a feedforward ReLU NN with $d$ input features and $L$ hidden layers with $n$ hidden neurons in each layer (with or without skip connections between different layers). Then the number of $d$-simplices in triangulations of all polytopes generated by $\mN$ is at most 
\begin{equation}
\frac{2n^{dL}}{(d-1)!(d!)^{L-1}} + \O(n^{dL-1}).
\end{equation}
In particular, if $L=1$, we derive the following upper bound for the maximum number of $d$-simplices 
$$
\#\text{simplices} \leq 
2n\sum_{i=0}^{d-1} \binom{n-1}{i}+2d \sum_{i=0}^{d-1} \binom{n}{i}. 
$$
\end{thm}

\begin{thm}[Lower Bound]
Let $\mN$ be a multi-layer fully-connected ReLU NN with $d$ input features and $L$ hidden layers with $n$ neurons in each layer. Then the maximum number of $d$-simplices in triangulations of polytopes generated by $\mN$ is at least $$\frac{n^{dL}}{d^{d(L-1)}d!}+\O(n^{dL-1}).$$ Furthermore, if $L=1$, we derive the following tighter lower bound for the maximum number of $d$-simplices $$
\#\text{simplices} \geq
\frac{2n}{d+1}\sum_{i=0}^{d-1} \binom{n-1}{i}.
$$
\label{thm:lower_bound_main_body}
\vspace{-0.4cm}
\end{thm}

\begin{proof}[Proof of Theorem \ref{thm:upper}] Directly by Theorems \ref{thm:1} and \ref{thm:m2}. 
\end{proof}

\begin{proof}[Proof of Theorem \ref{thm:lower_bound_main_body}] Directly by Theorems \ref{thm:2} and \ref{thm:m3}. 
\end{proof}

It is straightforward to see $(d-1)!(d!)^{L-1}< d^{d(L-1)}d!$; therefore, the above upper bound is strictly higher than the lower bound.
The basic idea to derive the above upper bound depends on the following observation: for each $(d-1)$-dim face of a $d$-dim polytope, it can only be a face for one unique simplex in a triangulation of this polytope, thus the total number of $d$-simplices in triangulations of polytopes must be smaller than or equal to the total number of $(d-1)$-dim faces in all polytopes. Therefore, we just need to derive the upper bound for the total number of $(d-1)$-dim faces in all polytopes generated by a neural network $\mN$, which can be done by induction on the number of layers of $\mN$. 
For the lower bound, we use the fact that each $d$-simplex with dimension $d$ has $d+1$ faces, thus the number of $d$-simplices should be at least the total number of $(d-1)$-dim faces in all polytopes divided by $d+1$.  

Our method to transfer the problems of calculating the above number of $d$-simplices to calculating the total number of $(d-1)$-dim faces in all polytopes is very versatile, and thus can be applied to many complicated architectures such as fully-connected NNs, CNNs, and ResNets \cite{he2016deep}. Actually, we can always calculate the total number of faces in all polytopes layer by layer, by considering each face and finding out how many new faces it is divided into by new hyperplanes from the next layer. 

Let's recall some basic knowledge on hyperplane arrangements \cite{stanley2004introduction}. Let $V$ be an Euclidean space. A hyperplane in the Euclidean space $V \simeq \mathbb{R}^n$, is a subspace $H := \{X \in V: \alpha \cdot X=b\}$,
where $\0\neq \alpha\in V$,  $b\in \mathbb{R}$ and $``\cdot"$ denotes the inner product. 
A {\em region} of an arrangement $\mA = \{H_i\subset V:1\leq i\leq m\}$ is just a connected component in the complement set of the union of all hyperplanes in the arrangement $\mA$. Let $r(\mA)$ be the number of regions for an arrangement $\mA$. Also, a \textit{simplex} in an $n$-dimensional Euclidean space is just a $n$-dimensional polytope that is the convex hull of $n + 1$ vertices. For example, a triangle is a simplex in $\mathbb{R}^2$, and a tetrahedron is a simplex in $\mathbb{R}^3$. A \textit{triangulation} on some polytope is a division of the polytope into simplices. 

The following Zaslavsky's Theorem is very crucial in the estimation of the number of linear regions. 
\begin{lem}[Zaslavsky's Theorem \cite{zaslavsky1975facing,stanley2004introduction}]\label{thm:ZaslavskyNN}
    Let $\mA$ be an arrangement with $m$ hyperplanes in
    $\R^{n}$. Then, the number $r(\mA)$ of regions for the arrangement $\mA$ satisfies
    \begin{eqnarray} \label{eq:region_general1}
        r(\mA)\leq\sum_{i=0}^{n} \binom{m}{i}.
\end{eqnarray} 
Furthermore, the above equality holds iff $\mA$ is in general position \cite{Stanley04anintroduction}.
\end{lem}

\textbf{Main results - One Layer ReLU NNs.}\label{sec: one-layer}
Throughout this paper, we always assume that the input space of an NN is a $d$-dimensional hypercube $C(d,B):= \{\x=(x_1,x_2,\ldots,x_d)\in \mathbb{R}^d: -B\leq x_i \leq B \}$ for some large enough constant~$B$.
Note that for a one-layer fully-connected ReLU NN, the pre-activation of each hidden neuron is an affine linear function of input values. Based on the sign of the pre-activation, each hidden neuron produces a hyperplane that divides the input space into two linear regions. On the other hand, the $d$-dimensional hypercube $C(d,B)$ has $2d$ hyperplanes in its boundary.

\begin{thm}\label{thm:1}
Let $\mN$ be a one-layer feedforward ReLU NN with $d$ input features and $n$ hidden neurons. Then the number of $d$-simplices in triangulations of polytopes generated by $\mN$ is at most 
$$
2n\sum_{i=0}^{d-1} \binom{n-1}{i}+2d \sum_{i=0}^{d-1} \binom{n}{i}  .
$$
\end{thm}
\begin{proof}
Let $H_1,H_2,\ldots,H_{n}$ be the $n$ hyperplanes generated by $n$ hidden neurons and $H_{n+1},H_{n+2},\ldots,H_{n+2d}$ be the $2d$ hyperplanes in the boundary of $C(d,B)$. Then for each $1\leq i \leq n$, the hyperplane $H_i$ may be intersected by other $n-1$ hyperplanes in $H_1,H_2,\ldots,H_{n}$. This will produce at most $n-1$ hyperplanes in $H_i$, thus by Theorem \ref{thm:ZaslavskyNN}, it will divide $H_i$ into at most $\sum_{i=0}^{d-1} \binom{n-1}{i}$ pieces since $H_i$ is a $(d-1)$-dim hyperplane.
Also, for each $1\leq i \leq 2d$, the hyperplane $H_{n+i}$ may be intersected by $H_1,H_2,\ldots,H_{n}$. This will produce at most $n$ $(d-2)$-dim hyperplanes in $H_{n+i}$, thus by Theorem \ref{thm:ZaslavskyNN}, it will divide $H_i$ into at most $\sum_{i=0}^{d-1} \binom{n}{i}$ pieces since $H_i$ is a $(d-1)$-dim hyperplane.
Moreover, each piece could be a face of two linear regions, finally, we will get at most 
$$
2n\sum_{i=0}^{d-1} \binom{n-1}{i}+2d \sum_{i=0}^{d-1} \binom{n}{i}  
$$ faces for all the polytopes. 
On the other hand, each simplex in a triangulation of polytope can be corresponding to at least one face in the polytope, and each face in the polytope can be corresponding to exactly one simplex. Therefore, the total number of $d$-simplices must be smaller than or equal to the total number of faces in all polytopes. Thus we obtain that the number of $d$-simplices in triangulations of polytopes generated by $\mN$ is also at most 
$$
2n\sum_{i=0}^{d-1} \binom{n-1}{i}+2d \sum_{i=0}^{d-1} \binom{n}{i}.
$$
\end{proof}

The following results give a lower bound for the maximum number of $d$-simplices in a triangulation of a one layer fully-connected ReLU NN.

\begin{thm}\label{thm:2}
Let $\mN$ be a one-layer fully-connected ReLU NN with $d$ input features and $n$ hidden neurons. If $n$ corresponding hyperplanes are in general position and $C(d,B)$ is large enough, then the number of $d$-simplices in a triangulation of polytopes among all $n$ corresponding hyperplanes is at least 
$$
\frac{2n}{d+1}\sum_{i=0}^{d-1} \binom{n-1}{i}=\frac{2n^d}{(d+1)(d-1)!}+\O(n^{d-1}).
$$
\end{thm}
\begin{proof}
Let $H_1,H_2,\ldots,H_{n}$ be $n$ hyperplanes generated by $n$ hidden neurons. Then for each $1\leq i \leq n$, the hyperplane $H_i$ will be intersected by other $n-1$ hyperplanes in $H_1,H_2,\ldots,H_{n}$. This will produce exact $n-1$ hyperplanes in $H_i$ since $H_1,H_2,\ldots,H_{n}$ are in general position, thus by Theorem \ref{thm:ZaslavskyNN}, it will divide $H_i$ into exact $\sum_{i=0}^{d-1} \binom{n-1}{i}$ pieces since $H_i$ is a $(d-1)$-dim hyperplane. When $C(d,B)$ is large enough, we can assume that every such a piece has a non-empty intersection with $C(d,B)$. Therefore, the total sum of number of $(d-1)$-faces of all linear regions (polytopes) will be at least $2n\sum_{i=0}^{d-1} \binom{n-1}{i}$ since every piece is counted twice. 
On the other hand, every $d$-dim simplex has $d+1$ distinct $(d-1)$-dim faces, thus every triangulation with $N$ simplices will contain $N(d+1)$ number $(d-1)$-dim faces. Therefore, if a triangulation of all linear regions (polytopes) of $\mN$ contains $N$ simplices, then
$$
N(d+1)\geq 2n\sum_{i=0}^{d-1} \binom{n-1}{i}
$$
and thus
$$
N\geq\frac{2n}{d+1}\sum_{i=0}^{d-1} \binom{n-1}{i}.
$$
Finally, we derive that a triangulation of all linear regions (polytopes) of $\mN$ contains at least $\frac{2n}{d+1}\sum_{i=0}^{d-1} \binom{n-1}{i}=\frac{2n^d}{(d+1)(d-1)!}+\O(n^{d-1})$ simplices.
\end{proof}



\textbf{Main results - Multi-Layer ReLU NNs.}\label{sec: multi-layer}
To study the multi-layer NNs, we need the following results from \cite[Proposation 3]{montufar2017notes}.
\begin{lem}[\cite{montufar2017notes}]\label{lem:m1}
Let $\mN$ be a multi-layer fully-connected ReLU NN with $d$ input features and $L$ hidden layers with $n_1,n_2,\ldots,n_L$ hidden neurons. Then the number of polytopes of $\mN$ is at most 
$\prod_{i=1}^{L}\sum_{j=0}^{m_i}  \binom{n_i}{j}$,
where $m_i=\min\{ d,n_1,n_2,\ldots,n_i \}$.
\end{lem}

\begin{thm}\label{thm:m2}
Let $\mN$ be a multi-layer feedforward ReLU NN with $d$ input features and $L$ hidden layers with $n$ hidden neurons in each layer (with or without skip connections between different layers). Then the number of $d$-simplices in triangulations of polytopes generated by $\mN$ is at most 
\begin{align}\label{eq:m2}
\frac{2n^{dL}}{(d-1)!(d!)^{L-1}} + \O(n^{dL}-1).
\end{align}
\end{thm}
\begin{proof}
First, we prove by induction that the total number of faces generated by $\mN$ is at most 
$$
\frac{2n^{dL}}{(d-1)!(d!)^{L-1}} + \O(n^{dL}-1).
$$
The case $L=1$ is proved in Theorem \ref{thm:1}. When $L\geq 2$, we assume that Eq. (\ref{eq:m2}) holds for $L-1$. Thus by Lemma \ref{lem:m1}, and the induction hypothesis, the network $\mN'$ with the first $L-1$ layers already has 
$$
\frac{n^{d(L-1)}}{(d!)^{L-1}} + \O(n^{d(L-1)-1})
$$
linear regions and 
$$
\frac{2n^{d(L-1)}}{(d-1)!(d!)^{L-2}} + \O(n^{d(L-1)-1})
$$
faces for all polytopes.  Then when we add the $L$-th layer, for each polytope $R$ with $f_R$ faces in $\mN'$, the $n$ neurons and the $f_R$ faces generate at most $n+f_R$ hyperplanes in $R$ (with or without skip connections between different layers, since the skip connections will not generate more hyperplanes or polytopes), similar to Theorem \ref{thm:1} these generates $$
2n\sum_{i=0}^{d-1} \binom{n-1}{i}+f_R \sum_{i=0}^{d-1} \binom{n}{i}  
$$ faces for all the polytopes in $R$. Therefore, we obtain that the total number of faces is at most
\begin{align*}
& ~~~~~~ 2n\sum_{i=0}^{d-1} \binom{n-1}{i} 
\cdot \left( \frac{n^{d(L-1)}}{(d!)^{L-1}} + \O(n^{d(L-1)-1}) \right) 
\\&~~~~~+ 
\sum_{i=0}^{d-1} \binom{n}{i}  \sum_{R} f_R
\\&=
2n\sum_{i=0}^{d-1} \binom{n-1}{i} 
\cdot \left( \frac{n^{d(L-1)}}{(d!)^{L-1}} + \O(n^{d(L-1)-1}) \right) 
\\&~~~~~+ 
\sum_{i=0}^{d-1} \binom{n}{i} \cdot  \left(\frac{2n^{d(L-1)}}{(d-1)!(d!)^{L-2}} + \O(n^{d(L-1)-1})\right)
\\&=
\frac{2n^{dL}}{(d-1)!(d!)^{L-1}} + \O(n^{dL}-1).
\end{align*}

Therefore, the total number of $d$-simplices must be smaller than or equal to the total number of faces in all polytopes. Thus we obtain that the number of $d$-simplices in triangulations of polytopes generated by $\mN$ is also at most 
$$
\frac{2n^{dL}}{(d-1)!(d!)^{L-1}} + \O(n^{dL}-1).
$$
\end{proof}

On the other hand, by the following lemma, it is easy to derive the maximum number of $d$-simplices in triangulations of polytopes generated by multi-layer NNs.

\begin{lem}[\cite{montufar2014number}]\label{lem:lower_bound_region}
Let $\mN$ be a multi-layer fully-connected ReLU NN with $d$ input features and $L$ hidden layers with $n_l$ hidden neurons in the $l$-th layer. Then the maximum number of linear regions of $\mN$ is at least $\prod_{l=1}^{L-1}\left\lfloor\frac{n_l}{d}\right\rfloor^{d}
\sum_{j=0}^{d}  \binom{n_L}{j}$.
\end{lem}

For the lower bounds, we have the following results.

\begin{thm}\label{thm:m3}
Let $\mN$ be a multi-layer fully-connected ReLU NN with $d$ input features and $L$ hidden layers with $n$ neurons in each layer. Then the maximum number of $d$-simplices in triangulations of polytopes generated by $\mN$ is at least $$\frac{n^{dL}}{d^{d(L-1)}d!}+\O(n^{dL-1}).$$ 
\end{thm}
\begin{proof}
By Lemma \ref{lem:lower_bound_region}, the maximum number of linear regions is lower bounded by $\left(\frac{n}{d}\right)^{d(L-1)}\sum_{i=0}^{d} \binom{n}{i}=\frac{n^{dL}}{d^{d(L-1)}d!}+\O(n^{dL-1})$. Also, the number of $d$-simplices should be larger than or equal to the number of linear regions. Thus we obtain the number of $d$-simplices in a triangulation of polytopes among all $n$ corresponding hyperplanes is at least $\frac{n^{dL}}{d^{d(L-1)}d!}+\O(n^{dL-1})$.
\end{proof}

We empirically validate our bounds in Table \ref{tab:verify_bound} with 4 structures. For a network structure X-$Y_1$-$\cdots$-$Y_h$-$\cdots$-$Y_H$-1, X represents the dimension of the input, and $Y_h$ is the number of hidden neurons in the $h$-th hidden layer. For a given MLP architecture, we initialize all the parameters based on the Xavier uniform initialization. Because all network structures we validate have a limited number of neurons, we can compute polytopes and their simplices by enumerating all collective activation states of neurons, which ensures that all polytopes are identifiable. For each structure, we repeat initialization ten times to report the maximum \#simplices. As shown in Table \ref{tab:verify_bound}, the derived upper bound is compatible with the numerical results of several network structures, which verifies the correctness of our results.

\begin{table}[htbp]
\centering
  \caption{Numerically verify the correctness of the derived upper and lower bounds for the maximum \#simplices.}
\scalebox{1}{\begin{tabular}{c|c|c|c|c}
\hline
        & 3-7-1 & 3-8-1 & 3-9-1 & 3-10-1\\ \hline
    Upper Bounds by Theorem \ref{thm:upper}  &   482 & 686 & 942 & 1256\\ \hline
    Enumeration Method & 446 & 663 & 893 & 1140\\ 
    \hline
    Lower Bounds by Theorem \ref{thm:lower_bound_main_body} & 77 & 116 & 166 & 230\\
 \hline
\end{tabular}}
\label{tab:verify_bound}
\end{table}


\textbf{Comparison of Different Network Architectures}\label{sec: comparison} Here, we compare the maximum $d$-\#simplices based on bounds obtained in the above. Our conclusion is that deep NNs usually have a larger number of $d$-simplices than shallow NNs with the same number of parameters.

First, let's fix some notations.
For two functions $f(n)$ and  $g(n)$, we write $f(n) = \Theta(g(n)) $ if there exists some positive constants $c_1,c_2$ such that $c_1 g(n) \leq f(n) \leq c_2 g(n)$ for all sufficiently large $n$; $f(n) = \O(g(n))$ if there exists some positive constant $c>0$ such that $f(n) \leq c g(n)$ for all sufficiently large $n$; and $f(n) = \Omega(g(n)) $ if there exists some positive constant $c$ such that $f(n) \geq c g(n)$ for all sufficiently large $n$.

The number of parameters for the fully-connected ReLU NN $\mN$ is easy to compute \cite[Proposition 7]{pascanu2013number}. 
\begin{lem}
\label{prop:number_params}
Let $\mN$ be a multi-layer fully-connected ReLU NN with $d$ input features and $L$ hidden layers with $n$ hidden neurons in each layer. Then the number of parameters in $\mN$ is
$
 \Theta(Ln^2).
$
\end{lem}

Let $S_{\mN_1}$ be the maximum number of $d$-simplices in triangulations of polytopes generated by $\mN$.
Now we can derive the number of $d$-simplices per parameter for deep NNs and their shallow counterparts. The following result follows directly from Lemma \ref{prop:number_params}, Theorem \ref{thm:upper} and Theorem \ref{thm:lower_bound_main_body}.
\begin{thm}\label{th:asy_compare}
Let $\mN_1$ be a multi-layer fully-connected ReLU NN with $d$ input features and $L$ hidden layers with $n$ hidden neurons in each layer, and $d=\O(1)$. Then $\mN_1$ has $\Theta (Ln^2)$ parameters, and the ratio of $S_{\mN_1}$ to the number of parameters of $\mN_1$ is 
$$
\frac{S_{\mN_1}}{\# \text{ parameters of } \mN_1}= 
\Omega \Bigl( \frac{1}{L} \cdot \frac{n^{dL-2}}{d^{d(L-1)}d!} \Bigr).
$$
For a one-layer fully-connected ReLU NN $\mN_2$ with $d$ input features and $Ln^2$ hidden neurons, it has $\Theta (Ln^2)$ parameters, and the ratio for $\mN_2$ is 
$$
\frac{S_{\mN_2}}{\# \text{ parameters of } \mN_2}=
\O\left( \frac{(Ln^2)^{d-1}}{(d-1)!}   \right).
$$
\end{thm}

From Theorem \ref{th:asy_compare} we obtain that $\frac{S_{\mN_1}}{\# \text{ parameters of } \mN_1}$ grows at least exponentially fast with the depth $L$ and polynomially fast with the width $n$.
In contrast, $\frac{S_{\mN_2}}{\# \text{ parameters of } \mN_2}$ grows at most polynomially fast with the numbers $L$ and $n$.

Therefore, we have that $\frac{S_{\mN_1}}{\# \text{ parameters of } \mN_1}$ is far larger than $\frac{S_{\mN_2}}{\# \text{ parameters of } \mN_2}$ when $L$ and $n$ are sufficiently large.
Thus we conclude that fully-connected ReLU NNs usually generate much more number of $d$-simplices than one-layer fully-connected ReLU NNs with asymptotically the same number of input dimensions and parameters. This result suggests that fully-connected ReLU NNs usually have much more expressivity than one-layer fully-connected ReLU NNs.

\subsection{Initialization}

\begin{thm}[One-hidden-layer NNs]\label{thm:sec6_1}
Let $\mN$ be a one-hidden-layer fully-connected ReLU NN with $d$ inputs and $n$ hidden neurons, where $d$ is a fixed positive integer. Suppose that $n$ hyperplanes generated by $n$ hidden neurons are in general position. Let $C(d,B):= [-B,B]^d$ be the input space of $\mN$ where $B$ is large enough. Then the average number of faces in linear regions of $\mN$ is at most $2d+\O(\frac{1}{n})$. In particular, when $n>2d^2+d$, the above bound becomes $2d+1$.
\label{thm:initial}
\end{thm}

\begin{proof}[Proof of Theorem \ref{thm:initial}]
By Theorem \ref{thm:upper}, we obtain that the number of $d$-simplices in triangulations of polytopes generated by $\mN$ is at most
$
\#\text{simplices} \leq 
2n\sum_{i=0}^{d-1} \binom{n-1}{i}+2d \sum_{i=0}^{d-1} \binom{n}{i}.
$
We know that the number of $(d-1)$-dim faces is no more than the number of $d$-simplices. On the other hand, since the $n$ hidden neurons are in general position and $B$ is large enough, we obtain that the total number of polytopes (i.e., linear regions) produced by $\mN$ is 
$\sum_{i=0}^{d} \binom{n}{i}.$
Therefore,  the average number of faces in linear regions of $\mN$ is at most
\begin{equation}
\begin{aligned}
     \frac{
2n\sum_{i=0}^{d-1} \binom{n-1}{i}+2d \sum_{i=0}^{d-1} \binom{n}{i}}{\sum_{i=0}^{d} \binom{n}{i}} \\
\leq 
\frac{
2n\sum_{i=0}^{d-1} \binom{n-1}{i}+2d \sum_{i=0}^{d-1} \binom{n}{i}}{\sum_{i=0}^{d-1} \binom{n}{i+1}}.
\label{Bound:1d}
\end{aligned}
\end{equation}
For each $0\leq i \leq d-1$, we have 
\begin{align*}
&\frac{2n\cdot\binom{n-1}{i}+2d\binom{n}{i}}{\binom{n}{i+1}} \\
\leq & 2(i+1) + \frac{2d(i+1)}{n-i}= 2(i+1) \left(1+\frac{d}{n-i}\right) \\\leq & 2d \left(1+\frac{d}{n-d+1}\right) =2d+\O(\frac{1}{n}).
\end{align*}
Therefore, the average number of faces in linear regions of $\mN$ is at most $2d+\O(\frac{1}{n})$. Furthermore,  when $n>2d^2+d$, the above bound becomes $2d+1$. 
\end{proof}

\begin{thm}[Multi-layer NNs, $d=2$]\label{thm:sec6_2}
Let $\mN$ be an $L$-layer fully-connected ReLU NN with $d=2$ inputs and $n_i$ hidden neurons in the $i$-th hidden layer. Let $C(d,B):= [-B,B]^d$ be the input space of $\mN$.  Furthermore, assume that $n_i$ and $B$ are large enough, then the average number of faces in linear regions of $\mN$ is at most $2d=4$.
\end{thm}

\begin{proof}
When $d=2$, the average number of faces can be naturally bounded. Let us start with a quadrilateral and add lines to it, then after adding one line in some linear region, the number of regions increases by $1$ and the total number of edges increases by at most $4$, thus the total number of edges is at most 4 times the number of linear regions, thus the average edge number is at most $4$ for the case $d=2$.
\end{proof}


\begin{thm}[Multi-layer NNs with Zero Biases]\label{thm:sec6_2_bias=0}
Let $\mN$ be an $L$-layer fully-connected ReLU NN with $d$ inputs and $n_i=n$ hidden neurons in the $i$-th hidden layer where $d$ and $n$ are two fixed positive integers. Suppose that all the biases of $\mN$ are equal to zero. Let $C(d,B):= [-B,B]^d$ be the input space of $\mN$.  Furthermore, assume that the number of hidden neurons and $B$ are large enough, then the average number of faces in linear regions of $\mN$ is at most $3d-2+\O(\frac{1}{n})$. In particular, there exists some constant  $C_d$ determined by $d$, such that when $n>C_d$, the above bound becomes $3d-1$.
\label{mNN_zero_biases}
\end{thm}

Let $\#\mA$ be the number of hyperplanes in an arrangement $\mA$ and $\rank(\mA)$ be the dimension of the space spanned by the normal vectors of the hyperplanes in $\mA$. An arrangement $\mA$ is called {\em central} if $\bigcap_{H\in\mA}  H\neq\emptyset$. Then we have the following results.

\begin{lem}[{Theorems 2.4 and 2.5 from \cite{Stanley04anintroduction}}]\label{lem:equality_hold}
Let $\mA$ be an arrangement in an $n$-dimensional vector space. Then we have  \[ r(\mA) = \sum_{\substack{\mB\subseteq\mA\\  \mB \text{ central}} } (-1)^{\#\mB-\rank(\mB)}. \] 
\end{lem}

\begin{lem}\label{lem:bias0}
Let $\mA$ be an arrangement with $m$ hyperplanes in
    $\R^{n}$. If all the hyperplanes in $\mA$ pass through the origin, and any $n$ normal vectors of $n$ hyperplanes in $\mA$ are linearly independent, then we have  \[ r(\mA) = \binom{m-1}{n-1} + \sum_{i=0}^{n-1} \binom{m}{i}. \] 
\end{lem}
\begin{proof}
Since all the hyperplanes in $\mA$ pass through the origin, then the intersection of all hyperplanes in $\mA$ is not empty, thus each $\mB\subseteq\mA$ must be central. 
Since any $n$ normal vectors of $n$ hyperplanes in $\mA$ are linearly independent, we have 
$$
\rank(\mB) = \min\{n,\#\mB\}.
$$
Therefore, by Lemma \ref{lem:equality_hold} we obtain
\begin{align*}
r(\mA) &= \sum_{\substack{\mB\subseteq\mA\\  \mB \text{ central}} } (-1)^{\#\mB-\rank(\mB)} 
\\&=
\sum_{i=0}^n \binom{m}{i} + \sum_{i=n+1}^m (-1)^{i-n} \binom{m}{i}
\\&=
\binom{m-1}{n-1} + \sum_{i=0}^{n-1} \binom{m}{i}.
\end{align*}
\end{proof}

By Lemma \ref{lem:bias0} we can derive the proof of Theorem \ref{mNN_zero_biases}.

\begin{proof}[Proof of Theorem \ref{mNN_zero_biases}]
Since all the biases of $\mN$ are equal to zero, then all the hyperplanes produced by the hidden neurons pass through the origin. Assume that the number of such hyperplanes is $n$ and they form an arrangement $\mA$. Then by Lemma \ref{lem:bias0} we obtain 
\begin{align*}
r(\mA) &= \binom{n-1}{d-1} + \sum_{i=0}^{d-1} \binom{n}{i} \\&= \frac{2}{(d-1)!} n^{d-1} + \O\left( n^{d-2}\right).
\end{align*}
On the other hand, for each $H\in \mA$, it will be intersected by other $n-1$ hyperplanes in $\mA$. This will produce $n-1$ hyperplanes in $H$, thus by Lemma \ref{lem:bias0}, it will divide $H$ into $\frac{2}{(d-2)!} n^{d-2} + \O\left( n^{d-3}\right)$ pieces since $H$ is a $(d-1)$-dim hyperplane. Similarly, for each hyperplane in the boundary of $C(d,B)$, it will be divided into 
$\frac{1}{(d-1)!} n^{d-1} + \O\left( n^{d-2}\right)$
pieces by $n$ hyperplanes in $\mA$.
Therefore, the total number of faces of linear regions formed by $\mA$ is at most
$$
2n\cdot\frac{2}{(d-2)!} n^{d-2} + 2d\cdot \frac{1}{(d-1)!} n^{d-1} + \O\left( n^{d-2}\right),
$$
which is equal to
$$
\left(\frac{4}{(d-2)!} +  \frac{2d}{(d-1)!}\right) n^{d-1} + \O\left( n^{d-2}\right).
$$
Finally, the average number of faces in linear regions of $\mN$ is at most $$
\frac{\left(\frac{4}{(d-2)!} +  \frac{2d}{(d-1)!}\right) n^{d-1} + \O\left( n^{d-2}\right)}{\frac{2}{(d-1)!} n^{d-1} + \O\left( n^{d-2}\right)} = 3d-2+\O(\frac{1}{n}).
$$
In particular, there exists some constant  $C_d$ determined by $d$, such that when $n>C_d$, the above bound becomes $3d-1$.
\end{proof}

\textbf{Remark 1. Interpretation of these bounds.}
Considering that $3d-1$ is a rather small bound, it can justify why simple polytopes dominate. If most polytopes are complex, the average face number should surpass $3d-1$ a lot. If simple polytopes only take up a small portion, the average face number will be larger than $3d-1$, too. In addition, unlike many other theories \cite{zhang2022neural,jacot2018neural,mei2018mean}, we do not assume that the network is infinitely wide in deriving the bound. 

Theorem \ref{mNN_zero_biases} and Theorems \ref{thm:sec6_1}, \ref{thm:sec6_2} are built for cases of zero biases and non-zero biases, respectively. It is a general practice to initialize biases with 0 before training a network, \textit{e.g.}, biases are often set to 0 in Xavier initialization \cite{glorot2010understanding}. Therefore, Theorem \ref{mNN_zero_biases} aligns with reality well. In addition, a ReLU network with zero biases becomes homogeneous, \textit{i.e.}, $\mN(\alpha\boldsymbol{\theta};\cdot)=\alpha^L\mN(\boldsymbol{\theta};\cdot)$, which is a widely-used setting when investigating implicit bias \cite{lyugradient,vardigradient}. Non-zero biases are so complicated to give a general and complete theorem for arbitrary cases. We only make success for one-hidden-layer networks with an arbitrary dimension and multi-layer networks with $d=2$. 

Yet looking straightforwardly, rigorously proving Theorems \ref{thm:sec6_1} and \ref{thm:sec6_2_bias=0} is intricate. The basic idea is twofold: Firstly, we derive the upper bound of simplices depending on the observation that for each $(d-1)$-dim face of a $d$-dim polytope, it can only be a face for one unique simplex in a triangulation of this polytope, thus the total number of $d$-simplices in triangulations of polytopes must be smaller than or equal to the total number of $(d-1)$-dim faces in all polytopes. Therefore, we just need to derive the upper bound for the total number of $(d-1)$-dim faces in all polytopes generated by a neural network $\mN$, which can be done by induction on the number of layers of $\mN$. Secondly, we derive the number of polytopes by the techniques and results from the classic hyperplane arrangement theories (see \cite{stanley2004introduction}). Finally, the quotient between the upper bound of simplices and the number of polytopes gives the upper bound for the average number of faces in linear regions of $\mN$.


\vspace{-0.3cm}
\subsection{After Training: Low-Rank}

\textit{Can we theoretically derive that polytopes remain simple after training?} It was shown that gradient descent-based optimization learns weight matrices of low rank \cite{galanti2023sgd, huh2021low, jigradient}. Therefore, under the low-rank setting, We also investigate if the polytopes are simple after the training. 
We derive Theorems \ref{mNN_zero_biases_low_rank} and \ref{mNN_zero_biases_one_hidden_layer_low_rank} to substantiate that after training, polytopes not just remain simple but turn simpler.

\begin{thm}[Multi-Layer NNs with Zero Biases and Low-rank Weight Matrices]\label{thm:sec6_2_bias=0'}
Let $\mN$ be an $L$-layer fully-connected ReLU NN with $d$ inputs and $n_i=n$ hidden neurons in the $i$-th hidden layer where $d$ and $n$ are two fixed positive integers. Assume that the weight matrix $W\in \mathbb{R}^{d\times n}$ in the first hidden layer has rank $d_0\leq d$. Suppose that all the biases of $\mN$ are equal to zero. Let $C(d,B):= [-B,B]^d$ be the input space of $\mN$.  Furthermore, assume that the number of hidden neurons and $B$ are large enough, then the average number of faces in linear regions of $\mN$ is at most $2d_0+d-2+\O(\frac{1}{n})$. In particular, there exists some constant  $C_d$ determined by $d$, such that when $n>C_d$, the above bound becomes $2d_0+d-1$.
\label{mNN_zero_biases_low_rank}
\end{thm}

\begin{proof}
The total number of faces of linear regions formed by $\mA$ is at most
$$
\left(\frac{4}{(d_0-2)!} +  \frac{2d}{(d_0-1)!}\right) n^{d_0-1} + \O\left( n^{d_0-2}\right).
$$
Finally, the average number of faces in linear regions of $\mN$ is at most $$
\frac{\left(\frac{4}{(d_0-2)!} +  \frac{2d}{(d_0-1)!}\right) n^{d_0-1} + \O\left( n^{d_0-2}\right)}{\frac{2}{(d_0-1)!} n^{d_0-1} + \O\left( n^{d_0-2}\right)}
= 2d_0+d-2+\O(\frac{1}{n}).
$$
In particular, there exists some constant  $C_d$ determined by $d$, such that when $n>C_d$, the above bound becomes $2d_0+d-1$.
\end{proof}

\begin{thm}[One-hidden-layer NNs, Low-rank Weight Matrices]\label{thm:sec6_1_training}
Let $\mN$ be a one-hidden-layer fully-connected ReLU NN with $d$ inputs and $n$ hidden neurons, where $d$ is a fixed positive integer. Assume that the weight matrix $W\in \mathbb{R}^{d\times n}$ has rank $d_0\leq d$, and any $d_0$ hyperplanes generated by any $d_0$ hidden neurons are in general position. Let $C(d,B):= [-B,B]^d$ be the input space of $\mN$. Furthermore, assume that $n$ and $B$ are large enough, then the average number of faces in linear regions of $\mN$ is at most $2d_0+\O(\frac{1}{n})$. In particular, when $n>2dd_0+d_0$, the above bound becomes $2d_0+1$.
\label{mNN_zero_biases_one_hidden_layer_low_rank}
\end{thm}

\begin{proof} 
We obtain that the number of $d$-simplices in triangulations of polytopes generated by $\mN$ is at most
$
\#\text{simplices} \leq 
2n\sum_{i=0}^{d_0-1} \binom{n-1}{i}+2d \sum_{i=0}^{d_0-1} \binom{n}{i}.
$
Also, the total number of polytopes produced by $\mN$ is 
$\sum_{i=0}^{d_0} \binom{n}{i}$ since any $d_0$ hyperplanes generated by any $d_0$ hidden neurons are in general position.
Therefore,  the average number of faces in linear regions of $\mN$ is at most
\begin{equation}
\begin{aligned}
  &   \frac{
2n\sum_{i=0}^{d_0-1} \binom{n-1}{i}+2d \sum_{i=0}^{d_0-1} \binom{n}{i}}{\sum_{i=0}^{d_0} \binom{n}{i}} \\
\leq &
\frac{
2n\sum_{i=0}^{d_0-1} \binom{n-1}{i}+2d \sum_{i=0}^{d_0-1} \binom{n}{i}}{\sum_{i=0}^{d_0-1} \binom{n}{i+1}}.
\end{aligned}
\label{Bound:1d}
\end{equation}
For each $0\leq i \leq d_0-1$, we have
\begin{equation}
    \begin{aligned}
& \frac{2n\cdot\binom{n-1}{i}+2d\binom{n}{i}}{\binom{n}{i+1}} \\
\leq & 2(i+1) + \frac{2d(i+1)}{n-i} 
\leq  2d_0   + \frac{2dd_0}{n-d_0+1} \\
=&2d_0+\O(\frac{1}{n}).
\end{aligned}
\end{equation}

Therefore, the average number of faces in linear regions of $\mN$ is at most $2d_0+\O(\frac{1}{n})$. Furthermore, when $n>2dd_0+d_0$, the above bound becomes $2d_0+1$.
\end{proof}

According to Theorems \ref{thm:sec6_2_bias=0'} and \ref{mNN_zero_biases_one_hidden_layer_low_rank}, we can see that when the weight matrix in the first hidden layer has a lower rank $d_0$, which is smaller than the input dimension $d$, then the average number of faces in linear regions of $\mN$ is mainly determined by $d_0$. 
This means that, after the training of a ReLU neural network, if the weight matrices become low-rank matrices (which is suggested by \cite{galanti2023sgd, huh2021low}), then the average number of faces in linear regions of $\mN$ would be much smaller, which means that the linear regions tend to be much simpler after training. 

\textbf{Remark 2. Explaining what happens when a network goes deep}. Modern deep learning theories, such as neural network Gaussian process \cite{zhang2022neural}, neural tangent kernel \cite{jacot2018neural}, and mean field \cite{mei2018mean}, need to assume infinite width, which essentially explain the behavior of a network when it goes wide. But depth is of most interest in the era of deep learning. Understanding of depth-induced network behaviors is essential in deciphering the mechanism of deep learning. However, currently, depth-oriented theories are few. Therefore, depth-oriented theories are a highly worthwhile research direction. Our theory from combinatorics suggests that increasing depth will not make the formed polytopes in ReLU networks more complex, which should be a valuable addition to the depth-oriented theory.

\textbf{Remark 3. Generalizing Implicit Bias}. It can be seen that the phenomenon that deep ReLU networks have simple polytopes is independent of the numerical method chosen for training, \textit{i.e.}, gradient descent. The result tends to be a constant only linearly dependent on the dimension. We still refer to this phenomenon as implicit bias by generalizing this concept. Our thinking is that the implicit bias can come from gradient descent and so on, and it can also be subjected to the geometric constraint by the network itself.

\vspace{-0.3cm}
\section{Conclusion}

In this manuscript, we have advocated studying the properties of polytopes instead of just counting them, towards revealing other valuable properties of a neural network. Then, we observed that deep ReLU networks have simple linear regions, which is not only a fundamental characterization but also explains what will happen when a ReLU network goes deep. Lastly, we have mathematically established a small bound for the average number of faces in polytopes, therefore supplying an explanation for the simple polytope phenomenon. An important future direction will be building the relationship between different forms of implicit biases \cite{galanti2023sgd}. If so, the understanding of implicit biases can be further deepened.


\bibliographystyle{ieeetr}
\bibliography{reference.bib}

\end{document}


\title{\huge\bf Supplementary Materials of “Deep ReLU Networks Have Surprisingly Simple Polytopes"}

\markboth{Journal of \LaTeX\ Class Files,~Vol.~14, No.~8, August~2021}%
{Shell \MakeLowercase{\textit{et al.}}: A Sample Article Using IEEEtran.cls for IEEE Journals}

	
\maketitle
	
\begin{abstract}

A ReLU network is a piecewise linear function over polytopes. Figuring out the properties of such polytopes is of fundamental importance for the research and development of neural networks. So far, either theoretical or empirical studies on polytopes only stay at the level of counting their number, which is far from a complete characterization. Here, we propose to study the shapes of polytopes via the number of faces of the polytope. Then, by computing and analyzing the histogram of faces across polytopes, we find that a ReLU network has relatively simple polytopes under both initialization and gradient descent, although these polytopes can be rather diverse and complicated by a specific design. This finding can be appreciated as a kind of generalized implicit bias, subjected to the intrinsic geometric constraint in space partition of a ReLU network. Next, we perform a combinatorial analysis to explain why adding depth does not generate a more complicated polytope by bounding the average number of faces of polytopes with the dimensionality. Our results concretely reveal what kind of simple functions a network learns and what will happen when a network goes deep. Also, by characterizing the shape of polytopes, the number of faces can be a novel leverage for other problems, \textit{e.g.}, serving as a generic tool to explain the power of popular shortcut networks such as ResNet and analyzing the impact of different regularization strategies on a network's space partition.

\end{abstract}
	
\begin{IEEEkeywords}
Deep Learning, ReLU Network, Polytopes, Complexity Analysis
\end{IEEEkeywords}
	
\section{Simplices, Polytopes, and Their Computation}
\label{app:computation_polytopes}

\subsection{Notations of a Network}
For convenience and consistency, we inherit notations from \citep{chu2018exact}. For a ReLU network that contains  $L$ hidden layers, we write the $l$-th layer as $\mathcal{L}_{l}$. Specially,  $\mathcal{L}_{0}$  is the input layer,  $\mathcal{L}_{L+1}$ is the output layer, and the other layers  $\mathcal{L}_{l}, l \in\{1, 2, \ldots, L\}$  are hidden layers. Hidden layers' neurons are called hidden neurons. Let  $n_{l}$  represent the number of neurons in  $\mathcal{L}_{l}$.

Given the $i$-th neuron of the $l$-th hidden layer, we denote by $\mathbf{b}_{i}^{(l)}$ its bias, by $\mathbf{a}_{i}^{(l)} $ its output, and by  $\mathbf{z}_{i}^{(l)}$  the total weighted sum of its inputs plus the bias. For all the  $n_{l}$  neurons in  $\mathcal{L}_{l}$, we arrange all their biases into a vector  $\mathbf{b}^{(l)}=\left[\mathbf{b}_{1}^{(l)}, \ldots, \mathbf{b}_{n_{l}}^{(l)}\right]^{\top}$, their outputs into a vector $\mathbf{a}^{(l)}=\left[\mathbf{a}_{1}^{(l)}, \ldots, \mathbf{a}_{n_{l}}^{(l)}\right]^{\top}$, and their inputs into a vector $\mathbf{z}^{(l)}=\left[\mathbf{z}_{1}^{(l)}, \ldots, \mathbf{z}_{n_{l}}^{(l)}\right]^{\top}$.
Neurons in two neighbor layers are linked by weighted edges. Denote by  $W_{ij}^{(l)}$  the weight of the edge between the  $i$-th neuron in  $\mathcal{L}_{l+1}$  and the  $j$-th neuron in  $\mathcal{L}_{l}$. For  $l \in\{1, \ldots, L\}$, we compute  $\mathbf{z}^{(l+1)}$  by
\begin{equation}
    \mathbf{z}^{(l+1)}=W^{(l)} \mathbf{a}^{(l)}+\mathbf{b}^{(l)},
    \label{eqn:layer_propagation}
\end{equation}
where $W^{(l)}$  is an  $n_{l+1}$-by-$n_{l}$  matrix.

Let $\sigma: \mathbb{R} \rightarrow \mathbb{R}$ be the ReLU activation function. We have  $\mathbf{a}_{i}^{(l)}=\sigma \left(\mathbf{z}_{i}^{(l)}\right)$  for all  $l \in\{1, \ldots, L\}$. In an element-wise fashion, we write  $\sigma\left(\mathbf{z}^{(l)}\right)=\left[\sigma\left(\mathbf{z}_{1}^{(l)}\right), \ldots, \sigma \left(\mathbf{z}_{n_{l}}^{(l)}\right)\right]^{\top}$. Then, for $l \in\{1, \ldots, L\}$, we have
\begin{equation}
\mathbf{a}^{(l)}=\sigma\left(\mathbf{z}^{(l)}\right).
\label{eqn:activation}
\end{equation}

The input is $\mathbf{x} \in \mathcal{X}$, where $\mathcal{X} \in \mathbb{R}^d$, and $\mathbf{x}_{i}$ is the $i$-th dimension of $\mathbf{x}$. The input layer $\mathcal{L}_{0}$  contains $n_{0}=d$ nodes, where $\mathbf{a}_{i}^{(0)}=\mathbf{x}_{i}, i \in\{1, \ldots, d\}$. The output of the network is $\mathbf{a}^{(L+1)} \in \mathcal{Y}$, where $\mathcal{Y} \subseteq \mathbb{R}^{n_{l+1}}$. The output layer $\mathcal{L}_{L+1}$  employs the softmax function:  $\mathbf{a}^{(L+1)}=\operatorname{softmax}\left(\mathbf{z}^{(L+1)}\right)$.

\subsection{Deriving Polytopes of a Network}

For the $i$-th hidden neuron in $\mathcal{L}_{l}$, $\sigma\left(\mathbf{z}_{i}^{(l)}\right)$  is in the following form:
\begin{equation}
\sigma\left(\mathbf{z}_{i}^{(l)}\right)=\left\{\begin{array}{cc}
\mathbf{z}_{i}^{(l)}, & \text{if}~ \mathbf{z}_{i}^{(l)} \geq 0 \\
0, & ~\text{if}~ \mathbf{z}_{i}^{(l)} <0,
\end{array}\right.
\label{eqn:sigma}
\end{equation}
where $\sigma\left(\mathbf{z}_{i}^{(l)}\right)$ consists of two linear parts. Given a network, an instance $\mathbf{x} \in \mathcal{X}$ determines the value of $\mathbf{z}_{i}^{(l)}$, and further determines  $\sigma\left(\mathbf{z}_{i}^{(l)}\right)=0$ or $\sigma\left(\mathbf{z}_{i}^{(l)}\right)=\mathbf{z}_{i}^{(l)}$. According to which part  $\sigma\left(\mathbf{z}_{i}^{(l)}\right)$ falls into, one can encode the activation status of each hidden neuron by two  states, each of which uniquely corresponds to one part of  $\sigma\left(\mathbf{z}_{i}^{(l)}\right)$. Denote by  $\mathbf{c}_{i}^{(l)} \in\{1, 0\}$  the state of the $i$-th hidden neuron in $\mathcal{L}_{l}$, we have $\mathbf{z}_{i}^{(l)} \geq 0$ if and only if $\mathbf{c}_{i}^{(l)}=1$ and $\mathbf{z}_{i}^{(l)} < 0$ if and only if $\mathbf{c}_{i}^{(l)}=0$. The states of different hidden neurons usually differ from each other.

Let $\mathbf{c}^{(l)}=\left[\mathbf{c}_{1}^{(l)}, \ldots, \mathbf{c}_{n_{l}}^{(l)}\right]$ be the states of all hidden neurons in $\mathcal{L}_{l}$ and $\mathbf{C}=\left[\mathbf{c}^{(1)}, \ldots, \mathbf{c}^{(L)}\right]$ specify the collective states of all hidden neurons. $\mathbf{C}$ of a given fixed network is uniquely determined by the instance $\mathbf{x}$. We write the function that maps an instance $\mathbf{x} \in \mathcal{X}$  to a configuration  $\mathrm{C} \in\{1, 0\}^{N}$  as $\mathrm{conf}: \mathcal{X} \rightarrow\{1, 0\}^{N}$, where $N=\sum_{i=0}^{L-1}n_l$.
Then, we rewrite Eq. (\ref{eqn:activation}) as
\begin{equation}
    \mathbf{a}^{(l)} = \sigma \left(\mathbf{z}^{(l)}\right)=\mathbf{c}^{(l)} \circ \mathbf{z}^{(l)},
\end{equation}
where  $\mathbf{c}^{(l)} \circ \mathbf{z}^{(l)}$  is the Hadamard product between $\mathbf{c}^{(l)}$  and  $\mathbf{z}^{(l)}$.

By plugging  $\mathbf{a}^{(l)}$  into Eq. (\ref{eqn:layer_propagation}), we rewrite  $\mathbf{z}^{(l+1)}$  as
\begin{equation}
    \mathbf{z}^{(l+1)}=W^{(l)}\left(\mathbf{c}^{(l)} \circ \mathbf{z}^{(l)}\right)+\mathbf{b}^{(l)}=\tilde{W}^{(l)} \mathbf{z}^{(l)}+\tilde{\mathbf{b}}^{(l)},
\label{eqn:iterative}    
\end{equation}
where $\tilde{\mathbf{b}}^{(l)}=\mathbf{b}^{(l)}$ , and  $\tilde{W}^{(l)}=W^{(l)} \circ \mathbf{c}^{(l)}$ is the generalized Hadamard product, such that the entry at the  $i$-th row and  $j$-th column of  $\tilde{W}^{(l)} $ is  $\tilde{W}_{i j}^{(l)}=W_{i j}^{(l)} \mathbf{c}_{j}^{(l)}$.

By iteratively enforcing Eq. (\ref{eqn:iterative}), we can write  $\mathbf{z}^{(l+1)},  l \in\{1, \ldots, L\}$  as
\begin{equation}
    \mathbf{z}^{(l+1)}=\left(\prod_{k=1}^{l} \tilde{W}^{(k)}\right) \mathbf{z}^{(1)}+\sum_{h=1}^{l} \left(\prod_{k=h+1}^{l} \tilde{W}^{(k)} \right)  \tilde{\mathbf{b}}^{(h)}.
\end{equation}

Substituting $\mathbf{z}^{(1)}=W^{(0)} \mathbf{x}+\mathbf{b}^{(1)}$ into the above equation, we rewrite  $\mathbf{z}^{(l+1)}, l \in\{1, \ldots, L\}$ as
\begin{equation}
\begin{aligned}
        & ~~~~~~~\mathbf{z}^{(l+1)} \\
        &=\left(\prod_{k=1}^{l} \tilde{W}^{(k)}\right) (W^{(0)} \mathbf{x}+\mathbf{b}^{(1)})+\sum_{h=1}^{l} \left(\prod_{k=h+1}^{l} \tilde{W}^{(k)} \right)  \tilde{\mathbf{b}}^{(h)} \\
        & = \left(\prod_{k=0}^{l} \tilde{W}^{(k)}\right) \mathbf{x}+\mathbf{b}^{(1)} \left(\prod_{k=1}^{l} \tilde{W}^{(k)}\right)+\sum_{h=1}^{l} \left(\prod_{k=h+1}^{l} \tilde{W}^{(k)} \right)  \tilde{\mathbf{b}}^{(h)} \\
        & = \hat{W}^{(0: l)}\mathbf{x}+ \hat{\mathbf{b}}^{(0: l)},
\end{aligned}
\end{equation}
where  $\hat{W}^{(0: l)}$ is the coefficient matrix of $\mathbf{x}$, and  $\hat{\mathbf{b}}^{(0: l)}$  is the sum of the remaining bias terms. Naturally, $F(\mathbf{x})$ is
\begin{equation}
F(\mathbf{x})=\operatorname{softmax}\left(\hat{W}^{(0:L)} \mathbf{x}+\hat{\mathbf{b}}^{(0:L)}\right).
\end{equation}

For a fixed network and a fixed instance  $\mathbf{x}, \hat{W}^{(0:l)}$  and  $\hat{\mathbf{b}}^{(0:l)}$  are constant parameters uniquely determined by activation states of all neurons: $\mathbf{C}=\operatorname{conf}(\mathbf{x})$. Furthermore, according to Eq. (\ref{eqn:sigma}), $\mathbf{C}$ leads to a group of inequalities:
\begin{equation}
    (2\mathbf{c}^{(l+1)}-1)\circ(\hat{W}^{(0: l)}\mathbf{x}+ \hat{\mathbf{b}}^{(0: l)})\geq 0,~l=0,\cdots,L-1,
\end{equation}
which encompasses a convex polytope according to the $H$-definition of polytopes (Chapter 16, \cite{toth2017handbook}). If $\mathrm{conf}(\x_1)=\mathrm{conf}(\x_2)$, $\x_1$ and $\x_2$ lie in the same polytope due to the uniqueness.

\subsection{Counting the \#Faces via Hit-and-Run Method}
Each inequality corresponds to a hyperplane. However, not all hyperplanes are faces of the encompassed polytope. \textit{We call the inequalities that are faces of the encompassed polytope non-redundant inequalities.}

Although for high dimensional problems, it's difficult to find all the non-redundant inequalities (faces of a polytope) in a short time, we can, however, apply probabilistic methods to find as many as possible to provide an effective estimation. There are various probabilistic methods to find necessary linear inequalities \cite{caron1997redundancy}. As Figure \ref{Figure_hit_and_run} shows, their basic idea is that one randomly "hits" the boundaries of the polytope from its inside. Here we simply use the classic Hit-and-Run algorithm \cite{berbee1987hit}. More specifically, we use the coordinate direction method, that is, we generate random directions along each coordinate axis with equal probability.

\begin{algorithm}
\label{alg: hit-and-run}
\begin{algorithmic}[1]
\STATE \textbf{Input}: $C(I)$, $x_0\in int(X(I))$
\STATE $j\leftarrow 0$\;
$D_0 \leftarrow \emptyset$\;
\STATE \textbf{Unless} Terminating Conditions
\STATE ~~~~~~~~~~~~Generate a random direction $v_j\in B(x_j,1)$
\STATE ~~~~~~~~~~~~\textbf{For}~{$i\in \{1,\dots, m\}$}
\STATE ~~~~~~~~~~~~~~~~~~{$\lambda_i\leftarrow \frac{b_i - a_i^Tx_j}{a_i^Tv_j}$} 
\STATE ~~~~~~~~~~~~~~~~~~$\lambda^+\leftarrow \min_{1\leq i\leq m}\{\lambda_i:\lambda_i > 0\}$\;
\STATE ~~~~~~~~~~~~~~~~~~$\lambda^-\leftarrow \max_{1\leq i\leq m}\{\lambda_i:\lambda_i < 0\}$\;
\STATE ~~~~~~~~~~~~Denote the indices of the minimum and the maximum above as $u_j^+$ and $u_j^-$, then $$D_{j+1} = D_j\cup\{u_j^+,u_j^-\};$$
\STATE ~~~~~~~~~~~~Generate $\sigma$ from a uniform distribution on $[0,1]$, and set
\STATE ~~~~~~~~~~~~$x_{j+1}\leftarrow x_j + (\lambda^- + \sigma(\lambda^+ - \lambda^-))v_j$;
\STATE ~~~~~~~~~~~~$j\leftarrow j+1$\;
\STATE \textbf{Output}: Indices of necessary inequalities, and indices of redundant inequities $D_j$, $I\/D_j$
\end{algorithmic}
\caption{Hit-and-Run algorithm}
\end{algorithm}

\begin{figure}[htb!]
\center{\includegraphics[width=\linewidth] {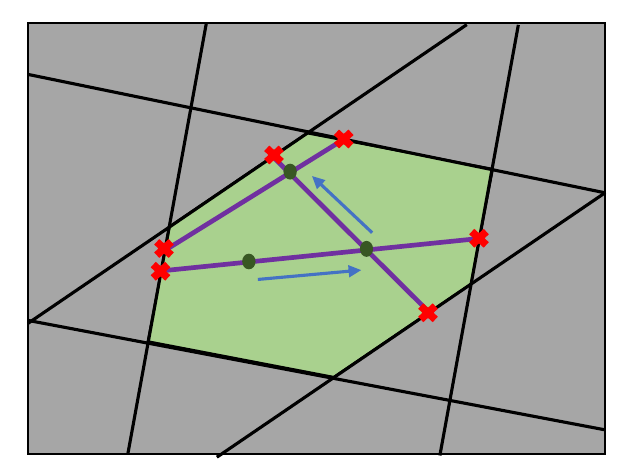}}
\caption{Hit-and-Run Algorithm}
\label{Figure_hit_and_run}
\end{figure}

\bibliographystyle{ieeetr}
\bibliography{reference.bib}